\documentclass{article}
\usepackage{microtype}
\usepackage{graphicx}
\usepackage{subcaption}
\usepackage{booktabs}
\usepackage{hyperref}
\usepackage[table, xcdraw]{xcolor}
\usepackage{wrapfig}
\usepackage{xurl}
\usepackage{algorithm}
\usepackage{algorithmic}

\usepackage{amsmath,amsthm,amssymb,bm}
\usepackage{amsfonts}
\usepackage{thmtools} 
\usepackage{bbm}
\usepackage{lipsum}
\usepackage{nicefrac}

\theoremstyle{plain}  
\newtheorem{theorem}{Theorem}[section]  
\newtheorem{lemma}[theorem]{Lemma}
\newtheorem{proposition}[theorem]{Proposition}
\newtheorem{corollary}[theorem]{Corollary}
\newtheorem{definition}[theorem]{Definition}

\newtheoremstyle{remarkstyle}
  {} 
  {} 
  {} 
  {} 
  {\bfseries} 
  {.} 
  {.5em} 
  {} 
  
\theoremstyle{remarkstyle} \newtheorem*{remark}{Remark}

\definecolor{mygray}{gray}{0.95}
\definecolor{mygray}{gray}{0.95}
\newcommand{\graybox}[1]{%
\begingroup
\setlength{\fboxsep}{0pt}%
\colorbox{mygray} {
\begin{minipage}{\linewidth}
\vspace{-0.5em}%
{#1}%
\end{minipage}%
}
\endgroup
}
\usepackage{color,soul}
\colorlet{mygreen}{green!55!black}
\definecolor{myyellow1}{HTML}{D89A3C}
\colorlet{myyellow}{myyellow1!90!black}
\colorlet{metablue}{blue!60!green}
\definecolor{nicerblue}{HTML}{417481} 

\usepackage{empheq}

    {%
        \end{gathered}\end{equation}
    }

\def\1{\bm{1}}

\DeclareMathAlphabet{\mathsfit}{\encodingdefault}{\sfdefault}{m}{sl}
\SetMathAlphabet{\mathsfit}{bold}{\encodingdefault}{\sfdefault}{bx}{n}

\usepackage{xspace}

\usepackage[normalem]{ulem}
\usepackage{multirow}
\usepackage{colortbl}      %

\definecolor{RoyalBlue}{rgb}{0.25, 0.41, 0.88}

\usepackage{pifont}
\usepackage{ifsym}
\colorlet{myred}{red!85!black}

\usepackage{enumitem} %
\usepackage[capitalise]{cleveref}

\usepackage[preprint]{icml2026}

\usepackage{amsmath}
\usepackage{amssymb}
\usepackage{mathtools}
\usepackage{amsthm}
\usepackage{empheq}
\usepackage{tcolorbox}

\crefname{equation}{}{}
\crefname{figure}{Fig.}{Figs.}
\crefname{section}{Sec.}{Secs.}
\crefname{appendix}{App.}{Apps.}
\crefname{table}{Tab.}{Tabs.}
\crefname{theorem}{Thm.}{Thms.}
\crefname{lemma}{Lem.}{Lems.}
\crefname{assumption}{Assump.}{Assumps.}
\crefname{proposition}{Prop.}{Props.}
\crefname{corollary}{Cor.}{Cors.}
\crefname{definition}{Def.}{Defs.}
\crefname{remark}{Rmk.}{Rmks.}
\crefname{algocf}{Alg.}{Algs.}
\crefname{algorithm}{Alg.}{Algs.}

\newenvironment{sketchofproof}{%
  \renewcommand{\qedsymbol}{}
  \begin{proof}[Sketch of proof]%
}{%
  \end{proof}%
}

\usepackage{color,soul}

\newcommand{\tp}{^\mathrm{T}}

\newcommand{\cE}{\mathcal{E}}

\newcommand{\cN}{\mathcal{N}}
\newcommand{\cX}{\mathcal{X}}
\renewcommand{\d}{\mathrm{d}}
\newcommand{\de}[2]{\frac{\mathrm{d}#1}{\mathrm{d}#2}} %

\newcommand{\mask}{\mathbf{M}}

\newcommand{\punif}{p_\mathrm{unif}}

\newcommand{\um}{\mathrm{UM}} %
\newcommand{\e}{\mathrm{e}}
\newcommand{\Z}{\mathbb{Z}}
\newcommand{\R}{\mathbb{R}}
\newcommand{\E}{\mathbb{E}}
\newcommand{\KL}{\mathrm{KL}}

\DeclareMathOperator*{\argmin}{argmin}

\newcommand{\unif}{\operatorname{Unif}}

\newcommand{\xb}{\overline{x}}

\newcommand{\gammab}{\overline{\gamma}}
\newcommand{\varphih}{\widehat{\varphi}}
\newcommand{\varPhih}{\widehat{\varPhi}}

\newcommand{\sigmat}{\widetilde{\sigma}}
\newcommand{\const}{\mathrm{const}}
\newcommand{\ro}[1]{\left(#1\right)}
\newcommand{\sq}[1]{\left[#1\right]}

\newcommand{\be}{\bm{e}}
\newcommand{\bR}{\bm{R}}
\newcommand{\bp}{\bm{p}}
\newcommand{\bI}{\bm{I}}
\newcommand{\ur}{u^\to}
\renewcommand{\ul}{u^\gets}
\newcommand{\one}{\mathbf{1}}
\renewcommand{\dh}{d_\mathrm{H}} %
\newcommand{\sg}{\operatorname{sg}} %

\newcommand{\tint}{{\textstyle\int\nolimits}}

\newcommand{\betah}{\beta_\mathrm{high}}
\newcommand{\betac}{\beta_\mathrm{critical}}
\newcommand{\betal}{\beta_\mathrm{low}}

\definecolor{colorcyc}{HTML}{cc0c3c}
\definecolor{colordm}{HTML}{01847F}
\definecolor{colortm}{HTML}{920872}

\newcommand{\colorcyc}[1]{{\color{colorcyc}#1}}
\newcommand{\colordm}[1]{{\color{colordm}#1}}
\newcommand{\colortm}[1]{{\color{colortm}#1}}

\icmltitlerunning{Discrete Adjoint Schrödinger Bridge Sampler}

\begin{document}
\onecolumn

  \icmltitle{Discrete Adjoint Schrödinger Bridge Sampler}

  \icmlsetsymbol{equal}{*}
  \icmlsetsymbol{equalsecond}{$\ddagger$}
  \icmlsetsymbol{equaladvising}{$\dagger$}

  \begin{icmlauthorlist}
    \icmlauthor{Wei Guo}{gt,equal}
    \icmlauthor{Yuchen Zhu}{gt,equalsecond}
    \icmlauthor{Xiaochen Du}{mit,equalsecond}
    \icmlauthor{Juno Nam}{mit,equalsecond}
    \icmlauthor{Yongxin Chen}{gt,equaladvising}
    \icmlauthor{Rafael Gómez-Bombarelli}{mit,equaladvising}
    \icmlauthor{Guan-Horng Liu}{fair,equaladvising}
    \icmlauthor{Molei Tao}{gt,equaladvising}
    \icmlauthor{Jaemoo Choi}{gt,equal}
  \end{icmlauthorlist}

  \icmlaffiliation{gt}{Georgia Institute of Technology}
  \icmlaffiliation{mit}{Massachusetts Institute of Technology}
  \icmlaffiliation{fair}{FAIR at Meta}

  \icmlcorrespondingauthor{Wei Guo}{wei.guo@gatech.edu}
  \icmlcorrespondingauthor{Jaemoo Choi}{jchoi843@gatech.edu}

  \icmlkeywords{Machine Learning, ICML}

  \vskip 0.3in

\printAffiliationsAndNotice{
\icmlEqualContribution
\icmlEqualContributionSecond
\icmlEqualAdvising}

\begin{abstract}
    Learning discrete neural samplers is challenging due to the lack of gradients and combinatorial complexity. While stochastic optimal control (SOC) and Schrödinger bridge (SB) provide principled solutions, efficient SOC solvers like adjoint matching (AM), which excel in continuous domains, remain unexplored for discrete spaces. We bridge this gap by revealing that the core mechanism of AM is \textit{state-space agnostic}, and introduce \textbf{discrete ASBS}, a unified framework that extends AM and adjoint Schrödinger bridge sampler (ASBS) to discrete spaces. Theoretically, we analyze the optimality conditions of the discrete SB problem and its connection to SOC, identifying a necessary cyclic group structure on the state space to enable this extension. Empirically, discrete ASBS achieves competitive sample quality with significant advantages in training efficiency and scalability.

\end{abstract}

\section{Introduction}
\label{sec:intro}

Sampling from unnormalized distributions is a fundamental problem across computational statistics \citep{liu2008monte,brooks2011handbook}, Bayesian inference \citep{gelman2013bayesian}, and statistical mechanics \citep{landau2014a}. The goal is to draw samples from a target distribution $\nu$ on a state space $\cX$, where $\nu(x)\propto\e^{-\beta E(x)}$ is specified through an energy function $E:\cX\to\R$ and inverse temperature $\beta>0$.
In high-dimensional settings with complex energy landscapes, traditional methods such as Markov chain Monte Carlo (MCMC) suffer from slow mixing and poor scalability.
Recent advances in \textbf{neural samplers} address these challenges by parameterizing sampling dynamics with deep neural networks, enabling efficient learning even without i.i.d. samples from the target. Crucially, these methods provide amortized inference, replacing costly MCMC iterations with rapid generation from a pretrained model.

For continuous state space $\cX=\R^D$, neural samplers have achieved remarkable success through diverse methodologies \citep{he2025no,sanokowski2025rethinking}, including sequential Monte Carlo \citep{phillips2024particle}, escorted transport \citep{vargas2024transport,albergo2025nets,chen2025sequential,blessing2025underdamped,du2025feat}, stochastic optimal control (SOC, \citet{zhang2022path,vargas2023denoising,richter2024improved}), parallel tempering \citep{rissanen2025progressive,akhoundsadegh2025progressive,zhang2026accelerated}, etc. Among SOC approaches, \textbf{adjoint matching} (\textbf{AM}, \citet{domingoenrich2025adjoint,havens2025adjoint}) has emerged as a powerful solver, enabling the framework of \textbf{adjoint Schrödinger bridge sampler} (\textbf{ASBS}, \citet{liu2025adjoint}), which offers fast convergence, scalability, and flexibility in reference dynamics. However, AM relies fundamentally on continuous calculus, making its extension to discrete state spaces highly nontrivial.

Motivated by these successes, recent work has explored discrete neural samplers based on continuous-time Markov chains (CTMCs), proposing similar methods based on escorted transport \citep{holderrieth2025leaps,ou2025discrete} and SOC formulations \citep{zhu2025mdns,guo2026proximal}. Despite this progress, the Schrödinger Bridge (SB) problem -- a distributionally constrained optimization framework central to optimal transport \citep{leonard2014a,chen2016relation,chen2021stochastic} and continuous diffusion models \citep{chen2021stochastic,de2021diffusion,chen2022likelihood,liu2022deep,shi2023diffusion} -- remains underdeveloped for discrete spaces \citep{ksenofontov2025categorical,kim2025discrete}. The theoretical connection between discrete SB and SOC is not fully established, and no discrete SOC solver achieves the efficiency of continuous AM. While a recent work \citet{so2026discrete} attempted to adapt AM to masked discrete diffusion, their approach yields complex training objectives, leaving a principled and efficient discrete AM framework an open challenge.

In this work, we develop a unified theoretical framework for SB and SOC on discrete state spaces, introducing the \textbf{discrete adjoint Schrödinger bridge sampler (DASBS)}. We derive optimality conditions for the discrete SB problem that structurally mirror the continuous setting, and extend AM to discrete domains through novel controller and corrector learning objectives. Our approach reveals that a group structure on the state space and uniform reference dynamics are essential -- requirements that parallel the additive noise assumption in continuous AM. Our \textbf{key contributions} are:

\textbf{I. Unified Discrete Framework}: We formalize discrete neural sampling as an SB problem for CTMCs and establish its equivalent SOC formulation.

\textbf{II. Generalizing Adjoint Matching}:
We identify state-space agnostic principles underlying AM, providing conceptual clarity for extensions beyond continuous spaces.

\textbf{III. Discrete ASBS}: We introduce discrete ASBS, a principled algorithm that alternates between adjoint and corrector matching, derived from variational characterizations of the learning objectives.

\textbf{IV. Empirical Validation}: We validate DASBS on high-dimensional discrete benchmarks, demonstrating competitive performance and efficiency.

\begin{table*}[t]
\centering
\caption{
Conceptual comparison between \colordm{denoising matching} and \colortm{adjoint matching} in continuous and discrete state spaces.
\colordm{Denoising matching} holds for general reference dynamics while \colortm{adjoint matching} requires \fbox{additive} noise. 
}
\label{tab:comparison_dm_tm}
\resizebox{\linewidth}{!}{
\begin{tabular}{ccc}
\toprule
 &
  Continuous state space $\cX=\R^D$ &
  Discrete state space $\cX=\Z_N^D$ \\ \midrule
Ref. dyn. $p^r$ &
  $\d X_t=b_t(X_t)\d t+\sigma_t\d W_t$, $X_0\sim\mu$ &
  CTMC $(X_t)$ with transition rate $(r_t)$, $X_0\sim\mu$ \\
Ctrl. dyn. $p^u$ &
  $\d X_t=(b_t+\sigma_tu_t)(X_t)\d t+\sigma_t\d W_t$, $X_0\sim\mu$ &
  CTMC $(X_t)$ with transition rate $(u_t)$, $X_0\sim\mu$ \\ \midrule
SB Prob. &
  $\min\limits_{u~\text{s.t.}~X_1\sim\nu}\underset{X\sim p^u}{\E}\Big[\int\limits_0^1\frac12\|u_t(X_t)\|^2\d t\Big]$ &
  $\min\limits_{u~\text{s.t.}~X_1\sim\nu}\underset{X\sim p^u}{\E}\Big[\int\limits_0^1\sum\limits_{y\ne X_t}\big(u_t\log\frac{u_t}{r_t}+r_t-u_t\big)(y,X_t)\d t\Big]$ \\
SOC Prob. &
  $\min\limits_{u}\underset{X\sim p^u}{\E}\Big[\int\limits_0^1\frac12\|u_t(X_t)\|^2\d t+\log\frac{\varphih_1}{\nu}(X_1)\Big]$ &
  $\min\limits_{u}\underset{X\sim p^u}{\E}\Big[\int\limits_0^1\sum\limits_{y\ne X_t}\big(u_t\log\frac{u_t}{r_t}+r_t-u_t\big)(y,X_t)\d t+\log\frac{\varphih_1}{\nu}(X_1)\Big]$ \\
 \midrule
Opt. Ctrl. &
  $u^\star_t(x)=\sigma_t\nabla\log\varphi_t(x)$ &
  $u^\star_t(y,x)=r_t(y,x)\frac{\varphi_t(y)}{\varphi_t(x)}$ \\
Corrector & $\nabla\log\varphih_1(x)$ &
  $\frac{\varphih_1(y)}{\varphih_1(x)}$ \\ \midrule
\begin{tabular}[c]{@{}c@{}}Example of\\ \fbox{Additive} Noise\end{tabular} &
  \begin{tabular}[c]{@{}c@{}}$b_t\equiv0\Rightarrow\boxed{p^r_{1|t}(x_1|x)=q_t(x_1-x)}$\\ $q_t(\varepsilon)=\cN(\varepsilon;0,\sigmat_t^2I)$\end{tabular} &
  \begin{tabular}[c]{@{}c@{}}$r_t(y,x)=\frac{\gamma_t}{N}1_{\dh(x,y)=1}\Rightarrow\boxed{p^r_{1|t}(x_1|x)=q_t(x_1-x)}$\\ $q_t(\varepsilon)=A(t,1)^{\dh(\varepsilon,0)}B(t,1)^{D-\dh(\varepsilon,0)}$\end{tabular} \\ \midrule
\begin{tabular}[c]{@{}c@{}}\textbf{\colordm{Denoising}} \\ \textbf{\colordm{Matching}}\end{tabular} &
  \begin{tabular}[c]{@{}c@{}}
  $\nabla\log\varphi_t(x)=\E_{p^\star_{1|t}(x_1|x)}\colordm{\nabla_x\log p^r_{1|t}(x_1|x)}$ \refstepcounter{equation}(\theequation)\label{eq:cts_ctrl_dm}
  \\
  $\nabla\log\varphih_1(x)=\E_{p^\star_{0|1}(x_0|x)}\colordm{\nabla_x\log p^r_{1|0}(x|x_0)}$ \refstepcounter{equation}(\theequation)\label{eq:cts_corr_dm}
  \end{tabular}
  &
  \begin{tabular}[c]{@{}c@{}}
  $\frac{\varphi_t(y)}{\varphi_t(x)}=\E_{p^\star_{1|t}(x_1|x)}\colordm{\frac{p^r_{1|t}(x_1|y)}{p^r_{1|t}(x_1|x)}}$ \\
  $\frac{\varphih_1(y)}{\varphih_1(x)}=\E_{p^\star_{0|1}(x_0|x)}\colordm{\frac{p^r_{1|0}(y|x_0)}{p^r_{1|0}(x|x_0)}}$
  \end{tabular}
  \\ \midrule
\begin{tabular}[c]{@{}c@{}}\textbf{\colortm{Adjoint}}\\ \textbf{\colortm{Matching}}\end{tabular} &
  \begin{tabular}[c]{@{}c@{}}
  $\nabla\log\varphi_t(x)=\E_{p^\star_{1|t}(x_1|x)}\colortm{\nabla\log\varphi_1(x_1)}$ \refstepcounter{equation}(\theequation)\label{eq:cts_ctrl_tm}\\
  $\nabla\log\varphih_1(x)=\E_{p^\star_{0|1}(x_0|x)}\colortm{\nabla\log\varphih_0(x_0)}$ \refstepcounter{equation}(\theequation)\label{eq:cts_corr_tm}
  \end{tabular}
  &
\begin{tabular}[c]{@{}c@{}}
  $\frac{\varphi_t(y)}{\varphi_t(x)}=\E_{p^\star_{1|t}(x_1|x)}\colortm{\frac{\varphi_1(y+x_1-x)}{\varphi_1(x_1)}}$ \\
  $\frac{\varphih_1(y)}{\varphih_1(x)}=\E_{p^\star_{0|1}(x_0|x)}\colortm{\frac{\varphih_0(y+x_0-x)}{\varphih_0(x_0)}}$
  \end{tabular} \\ \bottomrule
\end{tabular}
}
\end{table*}

\section{Preliminaries}
\label{sec:prelim}
\paragraph{Diffusion Samplers on Continuous State Space $\cX=\R^D$}
MCMC sampling based on equilibrium dynamics, e.g., Langevin Monte Carlo, is typically slow to mix when the target distribution is complex. Recent progress in non-equilibrium measure transport for generative modeling such as diffusion models \citep{song2021scorebased} has motivated sampling methods based on controlled stochastic differential equations (SDEs), commonly referred to as diffusion neural samplers. The sampling dynamics are described by SDE:
\begin{align}\label{eq:sde}
    p^u:~
    \d X_t = (b_t + \sigma_t u_t )(X_t) \d t + \sigma_t \d W_t,~X_0 \sim \mu, 
\end{align}
where $b: [0,1]\times\cX \rightarrow \cX$ is the base drift, $\sigma : [0,1] \rightarrow \R_{>0}$ is the noise schedule, and $\mu$ is the initial source distribution.
Given $(b_\cdot, \sigma_\cdot, \mu)$, the goal is to learn a parameterized control $u:[0,1]\times\cX\to\cX$ such that the marginal distribution of $X_1$ matches the target distribution $\nu$.
We use $p^u$ to denote the \textbf{path measure} induced by this SDE with control $u$. Formally, $p^u(X_{[0,1]})$ can be viewed as the limit of the joint distribution of $(X_{t_0},X_{t_1},...,X_{t_K})$ as the partition $0=t_0<t_1<...<t_K=1$ becomes finer and finer, and a rigorous definition is through the Radon-Nikod\'ym derivative (RND).

\paragraph{Schrödinger Bridge Problem on Continuous State Spaces}
One way to formulate the learning of diffusion samplers is through a distributionally constrained optimal transport problem known as the \textbf{Schrödinger bridge (SB)} problem \citep{leonard2014a,chen2016relation}:
\begin{align}\label{eq:sb_prob_cts}
    \min_{u~\text{s.t.}~X_1\sim\nu}\Big\{\KL(p^u\|p^r)=\underset{X \sim p^u}{\E}\int_0^1\frac{1}{2}\|u_t (X_t)\|^2 \d t\Big\},
\end{align}
where $p^u$ is the controlled path measure \cref{eq:sde} and $p^r$ is the path measure of a \textit{reference dynamics} with zero control ($u\equiv 0$).
The optimal control can be characterized by %
\begin{align}\label{eq:conti-u*}
    u^\star_t(x) = \sigma_t \nabla \log \varphi_t(x),
\end{align}
where the \textbf{SB potentials} $(\varphi_t,\varphih_t)$ are defined (up to multiplicative constants) through time integrations with respect to the reference transition kernel $p^r_{t|s}(y|x)$:
\begin{align*}
    \varphi_t(x) = \tint p^r_{1|t}(y|x) \varphi_1(y) \d y,
    & \quad \varphi_0(x)\varphih_0(x) = \mu(x); \\
    \varphih_t(x) = \tint p^r_{t|0}(x|y) \varphih_0(y) \d y,
    & \quad \varphi_1(x)\varphih_1(x) = \nu(x).
\end{align*}

\paragraph{Stochastic Optimal Control Characteristics of SB}
An interesting connection between the SB formulation and \textbf{stochastic optimal control (SOC)} is revealed by the characterization of the optimal control $u^\star$. As shown in \citet{dai1991stochastic,chen2016relation,liu2025adjoint}, $u^\star$ \cref{eq:conti-u*} is also the solution to the following SOC problem:
\begin{align}\label{eq:sb_soc_cts}
    \min_{u} \underset{X \sim p^u}{\E} \Big[
        \int_0^1 \frac{1}{2}\|u_t (X_t)\|^2 \d t + {\log \frac{\varphih_1(X_1)}{\nu(X_1)}}
    \Big].
\end{align}

\cref{eq:sb_soc_cts} highlights that SB \cref{eq:sb_prob_cts} admits an SOC interpretation, where the terminal marginal constraint is encoded through the terminal cost $\log \frac{\varphih_1}{\nu}$. This perspective provides a useful bridge between SB theory and SOC formulations.

\paragraph{Adjoint Schrödinger Bridge Sampler (ASBS, \citet{liu2025adjoint})}
When $b_t \equiv 0$, the reference dynamics reduces to Brownian motion. In this setting, the reference transition kernel is \textbf{additive} in the sense that sampling from $p^r_{1|t}(\cdot|x)$ can be achieved by adding a noise $\epsilon\sim q_t$ onto $x$, i.e.,
\begin{align}\label{eq:additive_noise_cts}
    p^r_{1|t}(y|x)=q_t(y-x),\quad q_t=\cN(0, \sigmat^2_t I),
\end{align}
where $\sigmat^2_t=\int_t^1\sigma^2_s\d s$.
Importantly, this additive property is central in simplifying the associated SOC problem \cref{eq:sb_soc_cts}; it yields tractable conditional distributions $p^r_{1|t}$ and enables explicit expressions for the quantities appearing in the optimality conditions \citep{havens2025adjoint}.
Exploiting this structure, \citet{liu2025adjoint} derived the \textbf{adjoint matching (AM)} identities
\cref{eq:cts_ctrl_tm,eq:cts_corr_tm},
as well as the corresponding \textbf{denoising matching (DM)} identities \cref{eq:cts_ctrl_dm,eq:cts_corr_dm}, and proposed an alternating procedure based on  \cref{eq:cts_ctrl_tm,eq:cts_corr_dm} to learn the controller and the corrector, which converges under suitable regularity conditions.
See the left part of \cref{tab:comparison_dm_tm} for details. 

\paragraph{Connection to the Memoryless Case (Adjoint Sampling)}
An earlier work, adjoint sampling \citep{havens2025adjoint}, considered a special case in which the reference path measure $p^r$ is \textbf{memoryless}, i.e. $p^r_{0,1}(x,y) = p^r_0(x)p^r_1(y)$,
under which one can show that $\varphih_1\propto p^r_1$. As a consequence, the corrector $\nabla\log\varphih_1$ is known and the learning problem reduces to a single regression objective for the controller.
From this perspective, ASBS is a generalization of adjoint sampling that relaxes the memoryless assumption, allowing for nontrivial boundary coupling. It is also observed in \citet{liu2025adjoint} that non-memoryless reference dynamics enable reduced noise levels compared with memoryless ones, thus offering improved performance.

\section{SB and SOC Theory for CTMC}
\label{sec:theory}

In this section, we introduce the SB problem for continuous-time Markov chains (CTMCs), serving as a discrete analog of \cref{eq:sb_prob_cts}. We then extend the corresponding SB and SOC theory from continuous state spaces to the discrete CTMC \citep{leonard2014a}. In particular, we derive optimality conditions for the optimal transition rate $u^\star$, analogous to \cref{eq:conti-u*}, and formulate an associated SOC problem in the spirit of \cref{eq:sb_soc_cts}.

\subsection{Problem Setting}
\label{sec:theory_prob_set}

\paragraph{Continuous-time Markov chain} Throughout this paper, we will consider the discrete state space $\cX$ as the set of length-$D$ sequences with $N$ possible states $[N]:=\{1,2,...,N\}$, i.e., $\cX=[N]^D$.
A \textbf{continuous-time Markov chain (CTMC)} $(X_t)_{t\in[0,1]}$ is a stochastic process taking values in $\cX$ and characterized by its \textbf{transition rate} $r=(r_t(y,x))^{x,y\in\cX}_{t\in[0,1]}$, defined by
\begin{align*}
    r_t(y,x)=\lim_{h\to0}\frac{\Pr(X_{t+h}=y|X_t=x)-1_{y=x}}{h},
\end{align*}
where $1_A\in\{0,1\}$ is the indicator of a statement $A$.

\paragraph{SB Problem on Discrete State Spaces} 
Following \cref{sec:prelim}, we address sampling problem by formulating a SB problem for CTMCs, whose optimal path measure $p^\star$ satisfies the boundary marginal constraints $p^\star_0 = \mu$ and $p^\star_1 = \nu$. Let $r_t(y,x)$ and $u_t(y,x)$ denote the reference and controlled transition rates, and write $p^r$ and $p^u$ for their induced \textbf{path measures}. We assume a common initial distribution $p^r_0 = p^u_0 = \mu$. 
Consider the following SB problem for CTMCs:

\graybox{
\begin{align}
    \min_{u~\text{s.t.}~p^u_1=\nu}\Big\{\KL(p^u\|p^r)=\underset{X\sim p^u}{\E}\!\int_0^1\!\sum_{y\ne X_t}\!\Big(u_t\log\frac{u_t}{r_t}+r_t-u_t\Big)(y,X_t)\d t\Big\}.
\label{eq:sb_prob}\tag{SB}
\end{align}
}

Next, we will make explicit the connection between this SB problem and a corresponding SOC formulation, in direct analogy with \cref{eq:sb_soc_cts}.

\subsection{SB and SOC Theory for CTMC}
\label{sec:theory_sb_soc}

\paragraph{Characterization of the Optimal Transition Rate}
We now characterize the optimal transition rate solving \cref{eq:sb_prob}. The following result provides an explicit description of the optimal transition rate in terms of a pair of time-dependent potentials, which play a role analogous to the SB potentials in the continuous-state setting. See \cref{app:theory_sb} for the proof.

\begin{theorem}
\label{thm:sb_sol}
The optimal transition rate $u^\star$ for \cref{eq:sb_prob} can be expressed as

\graybox{
\begin{align}
    u^\star_t(y, x)& = \frac{\varphi_t (y)}{\varphi_t (x)} r_t (y,x),~\forall y\ne x, \label{eq:u_star}
\end{align}
}

where the \textbf{SB potentials} $(\varphi_t, \varphih_t)$ satisfy

\graybox{
\vspace{0.5em}
\begin{empheq}[left={\forall 0\le s<t\le 1:~\empheqlbrace}]{align}
    \varphi_s(x)=\sum_y p^r_{t|s}(y|x)\varphi_t(y),\label{eq:varphi_kbe}\\
    \varphih_t(x)=\sum_y p^r_{t|s}(x|y)\varphih_s(y),\label{eq:hatvarphi_kfe}
\end{empheq}
}

and the optimal path measure $p^\star$ satisfies

\graybox{
\begin{align}\label{eq:sb_relation}
    p_t^\star(x)=\varphi_t(x)\varphih_t(x),\qquad
    \frac{p_{t|s}^\star(x|y)}{p^r_{t|s}(x|y)}=\frac{\varphi_t(x)}{\varphi_s(y)},\qquad
    \frac{p_{s|t}^\star(y|x)}{p^r_{t|s}(x|y)}=\frac{\varphih_s(y)}{\varphih_t(x)}.
\end{align}
}

\end{theorem}
\cref{eq:varphi_kbe,eq:hatvarphi_kfe} is a forward-backward representation with respect to the reference transition kernel, and the boundary marginal constraints are encoded through the coupling conditions at boundary times: $\varphi_0\varphih_0=\mu$, $\varphi_1\varphih_1=\nu$.

\paragraph{SOC Problem on Discrete State Spaces}
The SOC problem for CTMCs with terminal cost $g:\cX\to\R$ is

\graybox{
\begin{align}
    &\min_{u~\text{s.t.}~p^u_0=\mu}\Big\{\KL(p^u\|p^r)+\underset{X\sim p^u}{\E}g(X_1)\Big\}\label{eq:soc}\tag{SOC}
\end{align}
}

\paragraph{SOC Characteristics of SB} 
While the optimality conditions in \cref{eq:u_star} provide an explicit characterization of $u^\star$, they are challenging to solve in practice. The main difficulties are twofold: the coupled boundary constraints at times $t=0$ and $t=1$, and the need to evaluate expectations with respect to the reference transition kernels. The SOC formulation circumvents these issues by avoiding the direct solution of coupled equations. The following theorem establishes an SOC reinterpretation of the SB problem:

\graybox{
\vspace{1em}
\begin{theorem}\label{thm:sb_soc_relation}
    The optimal transition rate $u^\star_t$ \cref{eq:u_star} for \cref{eq:sb_prob_cts} solves \cref{eq:soc} with terminal cost $g\gets\log\frac{\varphih_1}{\nu}$.
\end{theorem}
}
\begin{sketchofproof}
The optimal path measures of \cref{eq:sb_prob,eq:soc} can be respectively written as
\begin{align}
    \frac{p^{\star}_{\text{SB}}(X_{[0,1]})}{p^r(X_{[0,1]})}&= \frac{\varphih_0(X_0)}{\mu(X_0)}\frac{\nu(X_1)}{\varphih_1(X_1)},\label{eq:sb_opt_path_measure}\\
    \frac{p^{\star}_{\text{SOC}}(X_{[0,1]})}{p^r(X_{[0,1]})}&=\frac{\e^{-g(X_1)}}{Z(X_0)},\quad Z(x)=\E_{p^r_{1|0}(y|x)}\e^{-g(y)}.\label{eq:soc_opt_path_measure}
\end{align}
\end{sketchofproof}
See \cref{app:theory_soc_sb} for the full proof.
\cref{thm:sb_soc_relation} shows that \cref{eq:sb_prob} admits an equivalent SOC formulation in which the terminal marginal constraint $p^u_1=\nu$ is incorporated through the terminal cost $g=\log\frac{\varphih_1}{\nu}$. This SOC perspective will be instrumental for developing tractable learning objectives in the discrete setting.

\paragraph{Connection to Memoryless Reference Dynamics}
If further assume $p^r$ is \textbf{memoryless},
then \cref{eq:varphi_kbe} implies $\varphi_0(x)=\E_{p^r_1}\e^{-g}=\const$, and \cref{eq:hatvarphi_kfe} implies $\varphih_1(x)=p^r_1(x)\sum_y\varphih_0(y)\propto p^r_1(x)$. Therefore, $g=\log\frac{p^r_1}{\nu}+\const$.

\paragraph{Relation to Continuous State Spaces}
Finally, we note that the SOC and SB theory developed above for discrete state spaces closely parallels its continuous counterpart introduced in \cref{sec:prelim}. A detailed comparison is provided in the upper part of \cref{tab:comparison_dm_tm}.

\section{Discrete Adjoint Schrödinger Bridge Sampler (DASBS)}
\label{sec:dasbs}

In this section, we introduce a principled theory and algorithm for learning the optimal transition rates in the discrete SB problem. Building on the SOC and SB formulation for CTMCs developed in the previous section, we propose a discrete analogue of ASBS \citep{liu2025adjoint} by developing discrete versions of adjoint matching and denoising matching for controller and corrector.

A key challenge in the discrete setting is the absence of additive noise structure \cref{eq:additive_noise_cts} that plays a central role in developing AM framework.
To address this issue, we will adopt a \textit{cyclic group structure} on the discrete state space, which allows discrete transitions to be interpreted in an additive form. Under this perspective, a uniform reference transition rate emerges as a natural choice for inducing tractable and symmetric transition kernels. 

\paragraph{Choice of the Reference Path Measure}
\cref{eq:u_star} implies that it suffices to learn the ratio of $\varphi_t$ at all pairs of $x,y\in\cX$ such that $r_t(y,x)>0$. We follow the typical strategy in discrete diffusion models \citep{campbell2022continuous,lou2024discrete,schiff2025simple} to restrict the transition to pairs of $x,y$ with Hamming distance $\dh(x,y)=1$, i.e., $x$ and $y$ differ at exactly one entry. Throughout this paper, we consider the reference path measure $p^r$ starting from an arbitrary tractable initial distribution $p^r_0=\mu$ and induced by the following \textbf{uniform transition rate} $r$ that keeps the uniform distribution on $\cX$ ($\punif=\frac{1}{N^D}$) invariant:
\begin{align}\label{eq:ref_transition_rate}
    r_t(y,x)=\begin{cases}
        \frac{\gamma_t}{N},&\text{if }\dh(y,x)=1,\\
        -\gamma_t D\ro{1-\frac1N},&\text{if }y=x,\\
        0,&\text{if otherwise,}
    \end{cases}    
\end{align}
where $\gamma_\cdot:[0,1]\to\R_+$ is a noise schedule. We remark that for two given time steps $0\le s<t\le1$ and states $x,y\in\cX$, $p^r_{t|s} (y|x)$ can be written in a closed form (\cref{prop:p_ref_trans_prob}), and so is $p^r_{t|0,1}(x|x_0, x_1)$ (\cref{prop:p_ref_bridge}). See \cref{app:theory_p_ref} for details.

Thus, it suffices to learn $\frac{\varphi_t(y)}{\varphi_t(x)}$ for all $\dh(x,y)=1$. Define the \textbf{controller} matrix $\varPhi^\star_t(x)\in\R^{D\times N}$ whose $(d,n)$-th element is $\varPhi^\star_t(x)_{d,n}=\frac{\varphi_t(x^{d\gets n})}{\varphi_t(x)}$, where $x^{d\gets n}$ denotes the vector obtained by replacing the $d$-th entry of $x$ with $n$.

\paragraph{Cyclic Group Structure for the State Space}
To extend AM \cref{eq:cts_ctrl_tm} to discrete space, we treat the state space $\cX=[N]^D$ as $\Z_N^D$, where $\Z_N$ is the cyclic group of integers modulo $N$. Thus, the sum and difference of any two elements in $\cX$ are still in $\cX$.
The key benefit of doing so is that the transition kernel $p^r_{1|t}(\cdot|x)$ is again \textbf{additive}:
\begin{align}\label{eq:additive_noise}
    \exists q_t \ \ \text{(see \cref{prop:p_ref_trans_prob}) \quad ~s.t.} \  \  \ p^r_{1|t}(y|x)=q_t(y-x).
\end{align}

\paragraph{Controller Adjoint Matching}
With an additive noise, we can thus consider a target score matching \citep{de2024target,zhang2025target} objective like in the continuous AM \citep{domingoenrich2025adjoint,havens2025adjoint}.
Applying \cref{eq:additive_noise}, we can refactor \cref{eq:varphi_kbe} as follows:
\begin{align}
    \varphi_t(y)
    &\underset{\cref{eq:varphi_kbe}}{=}
    \sum_{x_1\in\Z_N^D}p^r_{1|t}(x_1|y)\varphi_1(x_1) \underset{\cref{eq:additive_noise}}{=}
    \sum_{x_1\in\Z_N^D}p^r_{1|t}(x_1+\Delta|y+\Delta)\varphi_1(x_1) \nonumber\\
    &\underset{x'_1 \leftarrow x_1+\Delta}{=}
    \sum_{x'_1\in\Z_N^D}p^r_{1|t}(x'_1|y+\Delta)\varphi_1(x'_1 - \Delta) \underset{\Delta \leftarrow x-y}{=}
    \sum_{x'_1\in\Z_N^D}p^r_{1|t}(x'_1|x)\varphi_1(x'_1 - x + y),\label{eq:magic_varphi}\\
    \implies\frac{\varphi_t(y)}{\varphi_t(x)}&
    =\E_{p^r_{1|t}(x_1|x)}\frac{\varphi_1(x_1+\colorcyc{y-x})}{\varphi_t(x)} \underset{\cref{eq:sb_relation}}{=}
    \E_{p^\star_{1|t}(x_1|x)}\frac{\varphi_1(x_1+\colorcyc{y-x})}{\varphi_1(x_1)}.
    \label{eq:varphi_ratio_tm}
\end{align}

When $y\gets x^{d\gets n}$, since $\colorcyc{x^{d\gets n}-x}$ has at most one non-zero coordinate, $x_1$ and $x_1+\colorcyc{x^{d\gets n}-x}=x_1^{d\gets x_1^d+\colorcyc{n-x^d}}$ differ in at most one entry.
Let $D_f(a\|b)=f(a)-f(b)-(a-b)f'(b)\ge0$ be the Bregman divergence induced by a strictly convex and differentiable function $f$, and note that $\E\xi=\argmin_{c\in\R}\E D_f(\xi\|c)$ for any random variable $\xi$ \citep{lou2024discrete}.
We thus have the following variational characterization of the controller $\varPhi^\star$:
\begin{align}
    \varPhi^\star=\argmin_\varPhi\E_t\E_{p^\star_{t,1}(x,x_1)}
    \sum_{d=1}^D\sum_{n\ne x^d} D_f\ro{\frac{\varphi_1(x_1^{d\gets x_1^d+\colorcyc{n-x^d}})}{\varphi_1(x_1)}\middle\|\varPhi_t(x)_{d,n}},
    \label{eq:varphi_star_tsm}
\end{align}
where $t$ is a random variable in $(0,1)$ and one can sample $p^\star_{t,1}(x,x_1)$ by
\begin{align*}
    p^\star_{0,t,1}(x_0,x,x_1)&=p^\star_{0,1}(x_0,x_1)p^\star_{t|0,1}(x|x_0,x_1)\underset{\cref{eq:sb_sol_p_star_cond}}{=}p^\star_{0,1}(x_0,x_1)p^r_{t|0,1}(x|x_0,x_1).
\end{align*}
We can further rewrite the ratio in \cref{eq:varphi_star_tsm} as follows:
\footnote{For conciseness, $\triangle:=x_1^d+n-x^d$ and $\square:=x^d+n-x_1^d$.
\label{fn:shape}
}
\begin{align}
    \frac{\varphi_1(x_1^{d\gets \triangle})}{\varphi_1(x_1)}\underset{\cref{eq:sb_relation}}{=}\left.\frac{\nu(x_1^{d\gets \triangle})}{\nu(x_1)}\right/\left.\underbrace{\frac{\varphih_1(x_1^{d\gets \triangle})}{\varphih_1(x_1)}.}\right._{=\varPhih^\star (x_1)_{d,\triangle}}
    \label{eq:varphi1_ratio_param}
\end{align}
Notably, here, we require the \textbf{discrete score} $x\mapsto\ro{\frac{\nu(x^{d\gets n})}{\nu(x)}}_{d,n}$ of the target distribution, which is a \textit{first-order} oracle similar to the score of the target distribution in continuous AM.

\paragraph{Corrector Adjoint Matching}
From \cref{eq:varphi_star_tsm,eq:varphi1_ratio_param}, to learn the controller $\varPhi^\star_t$, we require estimating $\frac{\varphih_1(y)}{\varphih_1(x)}$ for all $\dh(x,y)=1$.
Let the \textbf{corrector} matrix $\varPhih^\star(x)\in\R^{D\times N}$ be defined by its entries $\varPhih^\star(x)_{d,n}=\frac{\varphih_1(x^{d\gets n})}{\varphih_1(x)}$.
Following the spirit of \cref{eq:magic_varphi}, we can derive the following identity and variational characterization of $\varPhih^\star$ (see \cref{app:theory_varphih_star_tsm} for proof):
\begin{align}\label{eq:hatvarphi_ratio_tm}
    \frac{\varphih_1(z)}{\varphih_1(y)}&=\sum_xp^\star_{t|1}(x|y)\frac{\varphih_t(x-y+z)}{\varphih_t(x)},~\forall t\in[0,1),\\
    \implies\varPhih^\star&=\argmin_{\varPhih}\E_{p^\star_{t,1}(x,x_1)}
    \sum_{d=1}^D\sum_{n\ne x^d}D_f\ro{\frac{\varphih_t(x^{d\gets x^d+n-x_1^d})}{\varphih_t(x)}\middle\|\varPhih(x_1)_{d,n}}.
    \label{eq:varphih_star_tsm}
\end{align}
However, if $t\ne0$, leveraging the relation \cref{eq:sb_relation} requires knowing the intractable $p^\star_t$; otherwise, as $\mu$ is known, and assume it is \textit{fully supported} on $\cX$, we have\footref{fn:shape}
\begin{align}
    \frac{\varphih_0(x^{d\gets\square})}{\varphih_0(x)}\underset{\cref{eq:sb_relation}}{=}
    \left.\frac{\mu(x^{d\gets\square})}{\mu(x)}\right/ \left.\underbrace{\frac{\varphi_0(x^{d\gets\square})}{\varphi_0(x)}.}\right._{=\varPhi^\star_0 (x)_{d,\square}}
    \label{eq:hatvarphi0_ratio_param}
\end{align}

\paragraph{Corrector Denoising Matching}
A limitation of \cref{eq:hatvarphi_ratio_tm,eq:hatvarphi0_ratio_param} is that they require the explicit density of $\mu$ to be positive everywhere. To circumvent this, one can leverage the following identity:
\begin{align}
    \frac{\varphih_1(z)}{\varphih_1 (y)}\underset{\cref{eq:hatvarphi_kfe}}{=}\sum_xp^r_{1|t} (z|x) \frac{\varphih_t (x)}{\varphih_1(y)}\underset{\cref{eq:sb_relation}}{=}\sum_x\frac{p^r_{1|t}(z|x)}{p^r_{1|t} (y|x)} p^\star_{t|1}(x|y).
    \label{eq:hatvarphi_ratio_dm}
\end{align}
Notably, while \cref{eq:hatvarphi_ratio_tm} relies on the additive noise \cref{eq:additive_noise}, \cref{eq:hatvarphi_ratio_dm} holds under general $p^r_{1|t}$. Thus, one can obtain a similar variational characterization of the corrector $\varPhih^\star$ by replacing the regression target (i.e., the ratio) in \cref{eq:varphih_star_tsm} with $\frac{p^r_{1|t}(x_1^{d\gets n}|x)}{p^r_{1|t}(x_1|x)}$. As this is related to the discrete score of the transition kernel $p^r_{1|t}(\cdot|x)$, we borrow the terminology in the continuous domain and call it \textbf{denoising matching (DM)}.

\paragraph{Alternating Update}
We can thus arrive at the core algorithm of \textbf{discrete ASBS}, following the practice in ASBS to learn $\varPhi\approx\varPhi^\star$ and $\varPhih\approx\varPhih^\star$ with an alternating update.
We initialize $\varPhih^{(0)}$ to be all-one following the practice in \citet{liu2025adjoint}, and for stage $k=1,2,...$, we solve the following two problems sequentially:

\graybox{
\begin{align}
    \varPhi^{(k)} &:= \argmin_\varPhi~
    \E_tw_t\E_{\substack{p^{\sg(r\varPhi)}_{0,1}(x_0,x_1)\\ p^r_{t|0,1}(x|x_0,x_1)}}
    \sum_{d=1}^D\sum_{n\ne x^d}D_f\ro{\colortm{\frac{\varphi_1(x_1^{d\gets x_1^d+n-x^d})}{\varphi_1(x_1)}}\middle\|\varPhi_t(x)_{d,n}},
    \label{eq:loss_varphi_tsm}\tag{ctrl-AM}\\
    \varPhih^{(k)} &:= \argmin_{\varPhih}~
    \E_{p^{\sg(r\varPhi^{(k)})}_{0,1}(x_0,x_1)}
    \sum_{d=1}^D\sum_{n\ne x_1^d}D_f\ro{\colortm{\frac{\varphih_0(x_0^{d\gets x_0^d+n-x_1^d})}{\varphih_0(x_0)}}\middle\|\varPhih(x_1)_{d,n}},
    \label{eq:loss_hatvarphi_tsm}\tag{corr-AM}\\
    \text{or}~~\varPhih^{(k)} &:=\argmin_{\varPhih}~
    \E_tw_t\E_{\substack{p^{\sg(r\varPhi^{(k)})}_{0,1}(x_0,x_1)\\ p^r_{t|0,1}(x|x_0,x_1)}}\sum_{d=1}^D\sum_{n\ne x_1^d}D_f\ro{\colordm{\frac{p^r_{1|t}(x_1^{d\gets n}|x)}{p^r_{1|t}(x_1|x)}}\middle\|\varPhih(x_1)_{d,n}}.\label{eq:loss_hatvarphi_dsm}\tag{corr-DM}
\end{align}
}

The stop gradient operator $\sg(\cdot)$ applied onto a model means not tracking the gradient when querying the model.
In \cref{eq:loss_varphi_tsm,eq:loss_hatvarphi_dsm}, $t\sim\unif(0,1)$, $w_\cdot:[0,1]\to\R_+$ is a time weight function. For the two AM losses, the \colortm{ratio} in \cref{eq:loss_varphi_tsm} is computed via \cref{eq:varphi1_ratio_param} by replacing $\varPhih^\star$ with the current $\sg(\varPhih^{(k-1)})$, and the \colortm{ratio} in \cref{eq:loss_hatvarphi_tsm} is computed via \cref{eq:hatvarphi0_ratio_param} by replacing $\varPhi^\star$ with the current $\sg(\varPhi^{(k)})$.
In all three losses, $p^{\sg(r\varPhi)}$ means sampling from the CTMC with transition rate $u_t(x^{d\gets n},x)=r_t(x^{d\gets n},x)\sg(\varPhi_t(x)_{d,n})$, $n\ne x^d$, which replaces $p^\star_{0,1}$ in the expectation with the detached non-optimal path measure. The theoretical justification of the validity of this replacement will be discussed in \cref{thm:cvg}, and in practice, an optional \textbf{trajectory importance reweighting} using the RND ${p^\star(x_{[0,1]})}/{p^{\sg(r\varPhi)}(x_{[0,1]})}$ can be incorporated (\cref{app:theory_imp_samp}).

\paragraph{Initialization}
Following ASBS \citep{liu2025adjoint}, one can initialize either the controller or the corrector to be non-informative (i.e., output all ones). In \cref{prop:networks_init}, we prove that under the reference transition rate \cref{eq:ref_transition_rate} and uniform initialization $\mu=\punif$, these two approaches are equivalent, and hence we always start with the all-one corrector.

\paragraph{Inference}
We use the \textbf{$\bm{\tau}$-leaping} method \citep{gillespie2001approximate,campbell2022continuous,lou2024discrete} to sample each dimension's transition independently. Since $u_t(x^{d\gets n},x)=\frac{\gamma_t}{N}\varPhi_t(x)_{d,n}$, $n\ne x^d$, this means we fix $\varPhi_\tau(X_\tau)_{d,n}$ on the interval $\tau\in[t,t+h]$ for a small step size $h>0$, and assume each dimension evolves independently. We defer the details to \cref{app:theory_tau_leap}.

\newcommand{\DeltaMag}{$\Delta\mathrm{Mag.}$}
\newcommand{\DeltaCorr}{$\Delta\mathrm{Corr.}$}
\newcommand{\EW}{$\mathrm{EW_2}$}
\definecolor{tabbg}{HTML}{DAE8FC}
\newcommand{\best}[1]{\mathbf{\uline{#1}}}
\newcommand{\sbest}[1]{\uline{#1}}

\begin{table*}[t]
\centering
\caption{
Learning to sample from lattice Ising models with $L=24$.
\textbf{\uline{Best}} and \uline{second best} results among all \textit{uniform-based} discrete neural samplers are highlighted.
$*$: Measured on one A6000 GPU with largest feasible batch size for $\betah$.
$\dagger$: For $\betac$ and $\betal$, using warm-up strategy in PDNS \citep{guo2026proximal}.
$\ddagger$: Failed to converge to meaningful distributions at $\betac$ and $\betal$ even with warm-up.
}
\label{tab:exp_ising}
\resizebox{\linewidth}{!}{
\begin{tabular}{ccccccccccccc}
\toprule
 & Inv. Temp. & & & \multicolumn{3}{c}{$\betah=0.28$} & \multicolumn{3}{c}{$\betac=0.4407$} & \multicolumn{3}{c}{$\betal=0.6$} \\ \cmidrule{2-13}
Type & Metrics $\downarrow$ & Steps ($\times1\e3)^*$ & Runtime (h)$^*$ & \DeltaMag & \DeltaCorr & \EW & \DeltaMag & \DeltaCorr & \EW & \DeltaMag & \DeltaCorr & \EW \\ \midrule
\rowcolor{tabbg}
\cellcolor{tabbg} & \textbf{DASBS} & $\best{3.75}$ & $\best{0.5}$ & $\sbest{2.7\e-3}$ & $\sbest{1.8\e-3}$ & $\sbest{8.6}$ & $\best{5.7\e-2}$ & $\best{4.7\e-2}$ & $\best{12.1}$ & $\best{2.0\e-2}$ & $\best{1.8\e-3}$ & $\best{2.0}$ \\
\rowcolor{tabbg}
\cellcolor{tabbg} & LEAPS & $\sbest{30}$ & $8.4$ & $\best{1.8\e-3}$ & $\best{9.2\e-4}$ & $\best{3.1}$ & $\sbest{5.9\e-2}$ & $\sbest{2.8\e-1}$ & $\sbest{96.5}$ & $\sbest{3.0\e-2}$ & $\sbest{5.5\e-1}$ & $\sbest{176.6}$ \\
\rowcolor{tabbg}
\cellcolor{tabbg} & UDNS$^\ddagger$ & $50$ & $11.9$ & $9.0\e-3$ & $8.7\e-3$ & $23.6$ & -- & -- & -- & -- & -- & -- \\ 
\rowcolor{tabbg}
\multirow{-4}{*}{\cellcolor{tabbg}Uniform} & DFNS$^\ddagger$ & $50$ & $\sbest{2.1}$ & $9.3\e-1$ & $8.0\e-1$ & $661.6$ & -- & -- & -- & -- & -- & -- \\ \midrule
Masked & MDNS$^\dagger$ & $50$ & $16.8$ & $3.9\e-3$ & $7.4\e-4$ & $0.1$ & $1.1\e-2$ & $5.6\e-3$ & $5.1$ & $9.0\e-3$ & $4.7\e-3$ & $5.3$ \\ \midrule
MCMC & MH & -- & -- & $8.9\e-4$ & $2.9\e-4$ & $1.2$ & $2.5\e-2$ & $3.7\e-3$ & $293.3$ & $4.0\e-2$ & $6.6\e-4$ & $109.9$ \\ \bottomrule
\end{tabular}
}
\end{table*}

\section{Additional Theory and Insights of DASBS}
\label{sec:additional_theory}
In this section, we provide further theory and insights to the DASBS framework.

\paragraph{AM v.s. DM for Corrector}
Recall that our derivation of the AM losses \cref{eq:loss_varphi_tsm,eq:loss_hatvarphi_tsm} leverages the relations \cref{eq:varphi_ratio_tm,eq:hatvarphi_ratio_tm}, which rely explicitly on the additiveness of the reference transition kernel $p^r_{1|t}(\cdot|x)$ \cref{eq:additive_noise}. In contrast, the DM loss \cref{eq:loss_hatvarphi_dsm} does not rely on this condition.

\paragraph{AM v.s. DM for Controller}
A parallel relationship holds for the \textit{controller},
which yields a different way of the alternating update by replacing the \colortm{ratio} in \cref{eq:loss_varphi_tsm} with $\colordm{\frac{p^r_{1|t}(x_1|x^{d\gets n})}{p^r_{1|t}(x_1|x)}}$ (see \cref{app:theory_dm,eq:loss_varphi_dsm} for details).
However, though theoretically grounded, this formulation may suffer from a weak \textit{mutual supervisory signal} during alternating updates: we rely on the trained $\varPhih^{(k-1)}$ to supervise the training of $\varPhi^{(k)}$, but such supervisory information only comes in from the \textit{implicit boundary relation} $\varPhi_1(x)_{d,n}\varPhih(x)_{d,n}=\frac{\nu(x^{d\gets n})}{\nu(x)}$, which is not explicitly enforced in the loss.
Even when the reference dynamics is \textit{memoryless} and there is no need to learn the corrector, we discover in \cref{fig:abl_tm_dm} that AM \cref{eq:loss_varphi_tsm} works significantly better than its DM counterpart \cref{eq:loss_varphi_dsm}.

\paragraph{Further Connection to Target Matching}
Such contrast resembles the distinction between \colortm{\textbf{target matching (TM)}} and \colordm{\textbf{denoising matching (DM)}} in the literature of learning scores of continuous probability distributions \citep{de2024target,kahouli2025control}: for two continuous random vectors $x,y\sim p(x,y)$, we can express the score as
\begin{align*}
    \underbrace{\E_{p(x|y)}\colortm{\nabla_x\log p(x)}}_{\text{for additive}~p(y|x)}=\nabla_y\log p(y)=\underbrace{\E_{p(x|y)}\colordm{\nabla_y\log p(y|x)}}_{\text{for general}~p(y|x)},
\end{align*}
where \textbf{additive} means $p(y|x)=q(y-x)$ for some distribution $q$. In other words, \colortm{TM} regresses onto the target score, whereas \colordm{DM} regresses onto the score of the transition kernel.
\colortm{TM} leverages a more meaningful learning signal and avoids the numerical instability of \colordm{DM} when the noise level is very small (where $\nabla_y\log p(y|x)$ becomes singular), thereby providing faster convergence of training, a benefit we observe directly in our discrete experiments (\cref{fig:abl_tm_dm}).

\paragraph{Unified View of Adjoint Matching}
\textbf{Adjoint matching} (\textbf{AM}, \citet{domingoenrich2025adjoint}) is originally derived by analyzing the ODE of the \textit{adjoint state} -- the gradient of the cost-to-go from $X_t=x$ with respect to $x$. This formulation connects the adjoint state to the gradient of the objective with respect to control parameters, yielding a learning objective whose \textbf{unique fixed-point} corresponds to the optimal control. However, such gradient-based perspective is not directly generalizable to discrete domains.
Recently, \citet{so2026discrete} made a first attempt by expressing the optimal control as an expectation under the optimal path measure; however, their reliance on \textit{masked} transition rate resulted in a complex training objective, due to the lack of the additive property.
In contrast, our derivation identifies that the \textbf{additive reference noise} is the key structural requirement that enables a TM-like loss in discrete domains.
We refer readers to the lower half of \cref{tab:comparison_dm_tm} for a side-by-side comparison of DM and AM formulations across continuous and discrete settings, and conclude with a unified characterization of the intrinsic nature of AM:

\begin{tcolorbox}[
colback=gray!10, colframe=gray!10,
sharp corners,
left=0.2em, right=0.2em,
top=0.3em, bottom=0.3em,
before skip=0.3em, after skip=0.3em
]
    \centering
    \textit{
    Adjoint matching:
    a \textbf{fixed-point iteration} driven by a \textbf{\colortm{target matching}} objective converging to the \textbf{optimal} $\boldsymbol{p^\star}$.
    }
\end{tcolorbox}

\paragraph{DASBS Unifies Existing  Memoryless SOC Solvers}
We further establish the connection between DASBS and the \textbf{weighted denoising cross-entropy (WDCE)} loss for solving \textit{memoryless} SOC problems \citep{zhu2025mdns,zhu2025enhancing} (see \cref{prop:udns,prop:mdns} for full statement and proof):
\begin{proposition}
    \label{prop:udns_mdns_informal}
    Under a \emph{memoryless} reference path measure $p^r$ (encompassing both uniform \cref{eq:ref_transition_rate} and masked \cref{eq:ref_transition_rate_mask} discrete diffusion), the denoising loss for the controller \cref{eq:loss_varphi_dsm}, with trajectory importance reweighting, generalized KL divergence,
    \footnote{$f(t)=t\log t\implies D_f(a\|b)=a\log\frac{a}{b}-a+b$. \label{fn:gen_kl}}
    and time weight $w_t\gets\frac{\gamma_t}{N}$, is equivalent to the WDCE loss.
\end{proposition}

\paragraph{Convergence Analysis}
Finally, following the continuous arguments \citep[Thm. 4]{liu2025adjoint}, we establish the following convergence guarantee of DASBS.

\begin{theorem}
\label{thm:cvg}
\textbf{(1)} The path measure induced by the unique fixed-point of \cref{eq:loss_varphi_tsm} ($\varPhi^{(k)}$) solves a \textbf{forward half bridge problem} $\min_{p~\text{s.t.}~p_0=\mu}\KL(p\|q^{\varphih^{(k-1)}_1})$ for some path measure $q^{\varphih^{(k-1)}_1}$ induced by $\varphih^{(k-1)}_1$ (\cref{def:cvg_q_psi}).

\textbf{(2)}
The unique fixed-point of \cref{eq:loss_hatvarphi_dsm} is the same as the unique fixed-point of \cref{eq:loss_hatvarphi_dsm}. Denote it as $\varPhih^{(k)}$, which induces a path measure $q^{\varphih^{(k)}_1}$ that solves a \textbf{backward half bridge problem} $\min_{q~\text{s.t.}~q_1=\nu}\KL(p^{r\varPhi^{(k)}}\|q)$.
\end{theorem}
Consequently, convergence is guaranteed by the theory of iterative proportional fitting \citep{ruschendorf1995convergence,chen2016entropic,de2021diffusion}.
See \cref{app:theory_cvg} for the proof.

\begin{table*}[t]
\centering
\caption{
Learning to sample from lattice Potts models with $L=16$ and $N=4$.
\textbf{\uline{Best}} result among all \textit{uniform-based} discrete neural samplers are highlighted.
$\dagger$: For $\betac$ and $\betal$, using warm-up strategy in PDNS \citep{guo2026proximal}.
}
\label{tab:exp_potts}
\resizebox{\linewidth}{!}{
\begin{tabular}{ccccccccccc}
\toprule
 & Inv. Temp. & \multicolumn{3}{c}{$\betah=0.9$} & \multicolumn{3}{c}{$\betac=1.0986$} & \multicolumn{3}{c}{$\betal=1.3$} \\ \midrule
Type & Metrics $\downarrow$ & \DeltaMag & \DeltaCorr & \EW & \DeltaMag & \DeltaCorr & \EW & \DeltaMag & \DeltaCorr & \EW \\ \midrule
\rowcolor{tabbg}
\cellcolor{tabbg} & \textbf{DASBS} & $8.7\e-3$ & $6.6\e-3$ & $\best{5.3}$ & $\best{4.1\e-3}$ & $\best{2.9\e-2}$ & $\best{20.0}$ & $\best{4.2\e-3}$ & $\best{7.3\e-3}$ & $\best{3.6}$ \\
\rowcolor{tabbg}
\multirow{-2}{*}{\cellcolor{tabbg}Uniform} & LEAPS & $\best{5.7\e-3}$ & $\best{5.0\e-3}$ & $8.2$ & $3.2\e-1$ & $2.6\e-1$ & $79.9$ & $3.6\e-1$ & $3.5\e-1$ & $90.5$ \\ \midrule
Mask & MDNS & $6.5\e-3$ & $5.1\e-3$ & $6.4$ & $5.2\e-3$ & $4.6\e-3$ & $1.9$ & $8.4\e-4$ & $6.1\e-4$ & $0.7$ \\ \midrule
MCMC & MH & $5.2\e-2$ & $4.7\e-2$ & $98.6$ & $4.9\e-1$ & $4.0\e-1$ & $273.2$ & $6.8\e-1$ & $6.4\e-1$ & $313.1$ \\ \bottomrule
\end{tabular}
}
\end{table*}

\section{Experiments}
\label{sec:exp}

In this section, we evaluate the effectiveness and efficiency of the proposed DASBS algorithm. We demonstrate that it achieves competitive sample quality across standard benchmarks while offering significant advantages in training speed. We further investigate the benefits of AM and non-memoryless reference dynamics through ablation studies. Additional details and results are provided in \cref{app:exp}.

\paragraph{Experimental Setup}
We test our method on two discrete distributions from statistical physics: the Ising and Potts models on square lattices. They are well known for exhibiting phase transitions as a function of the inverse temperature \citep{onsager1944crystal,beffara2012self} and serve as standard benchmarks for discrete sampling.
For the Ising model, we use a lattice of size $L=24$ (i.e., sequence length $D=L^2=576$) with states $\{\pm1\}$, and consider inverse temperatures $\betah=0.28$, $\betac=\log(1+\sqrt{2})/2\approx0.4407$, and $\betal=0.6$.
For the Potts model, we use $L=16$ ($D=L^2=256$) with $N=4$ states, and consider $\betah=0.9$, $\betac=\log(1+\sqrt{q})\approx1.0986$, and $\beta=1.3$.
We benchmark DASBS against leading discrete neural samplers, including LEAPS \citep{holderrieth2025leaps}, DFNS \citep{ou2025discrete}, UDNS \citep[App. F]{zhu2025mdns}, and MDNS \citep{zhu2025mdns}, as well as the classical Metropolis-Hastings (MH) algorithm.

\paragraph{Results and Discussion}
Quantitative results for the Ising and Potts models are summarized in \cref{tab:exp_ising,tab:exp_potts}, respectively. We assess sample quality using three metrics: the deviation of the magnetization (\DeltaMag), the deviation of the 2-point correlation (\DeltaCorr), and the energy Wasserstein-2 distance (\EW), all relative to ground-truth samples from the Swendsen-Wang (SW, \citet{swendsen1986replica,swendsen1987nonuniversal}) algorithm. For Ising model with high temperature, we also report the number of training steps and runtime on a single A6000 GPU, utilizing the maximum feasible batch size. 
Across all regimes, DASBS delivers \textbf{satisfactory sample fidelity} that is highly competitive with state-of-the-art uniform-based discrete neural samplers. Notably, DASBS confers a distinct advantage in \textbf{training efficiency}. We attribute this speed-up to three factors: (1) the use of the discrete score as a highly informative first-order oracle; (2) a memory-efficient design that avoids storing full roll-outs or computing loss over full trajectories; and (3) the simplicity of the matching loss objective. 

\begin{figure}[t]
    \centering
    \includegraphics[width=0.75\linewidth]{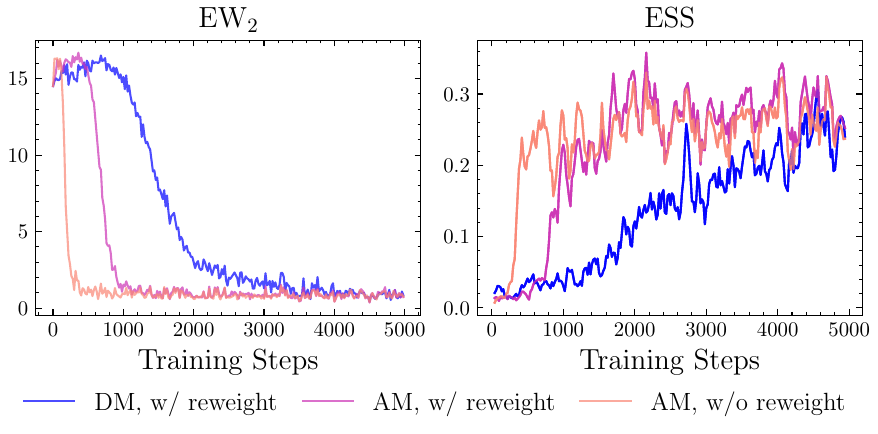}
    \caption{
    Ablation study of the adjoint matching (AM) and denoising matching (DM) training losses for the \textit{memoryless} noise schedule $\gamma_t=\frac1t$ on Potts model with $L=8$, $N=3$, $\betah=0.5$.
    Reweighting means using trajectory importance weight $p^\star/p^{\sg(u)}$ in the training losses.
    DM with reweighting corresponds to UDNS \citep[App. F]{zhu2025mdns}.
    \textit{Left}: Energy Wasserstein-2 distance to the ground-truth samples from SW algorithm.
    \textit{Right}: Effective sample size computed from the trajectory importance weights.
    }
    \label{fig:abl_tm_dm}
\end{figure}

\begin{figure}[t]
    \centering
    \includegraphics[width=0.75\linewidth]{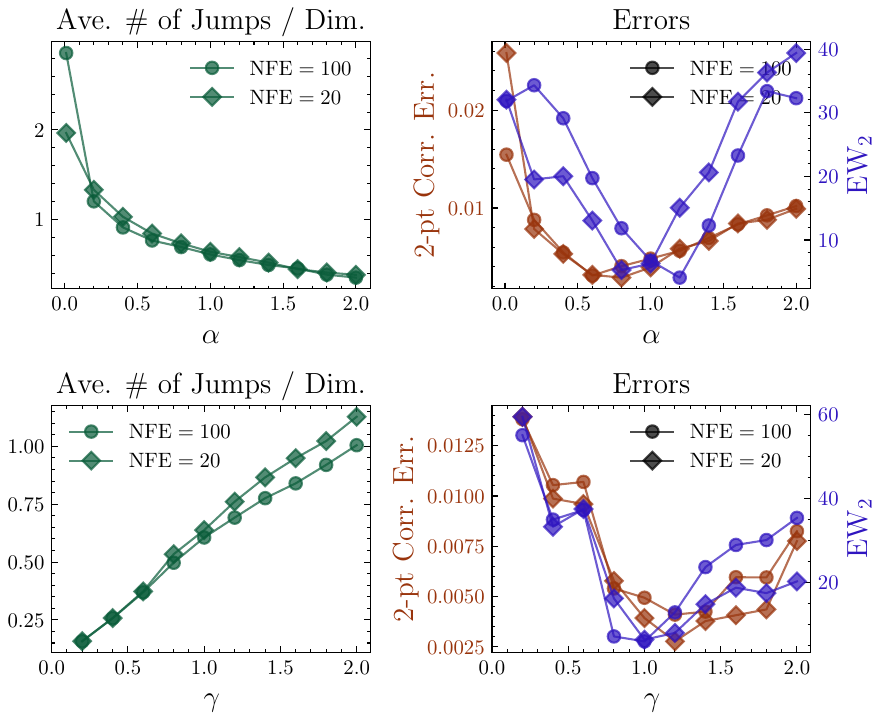}
    \caption{
    Ablation study of the hyperparameters $\alpha$ and $\gamma$ for the modified log-linear noise schedule $\gamma_t=\frac{\gamma}{t+\alpha}$ on Ising model with $L=24$ and $\betah=0.28$.
    The case of $\alpha=0$ is \textit{memoryless}.
    NFE is the number of function evaluations during generation for both training and inference.
    \textit{Top}: fix $\gamma=1$ and vary $\alpha$.
    \textit{Bottom}: fix $\alpha=1$ and vary $\gamma$.    
    \textit{Left}: average number of jumps for each dimension during generation.
    \textit{Right}: 2-point correlation error and energy Wasserstein-2 distance to ground-truth samples drawn from SW algorithm.
    }
    \label{fig:abl_loglin}
\end{figure}

\paragraph{Ablation Study: AM v.s. DM in Learning Controller}
In \cref{fig:abl_loglin}, we empirically validate the claim that AM is more efficient than DM. We focus on an $8\times 8$ Potts model ($N=3$, $\betah=0.5$). We choose the standard \textit{memoryless} noise schedule $\gamma_t=\frac{1}{t}$ \citep{ou2025your,zhu2025mdns} to isolate the impact of the training objectives \cref{eq:loss_varphi_tsm} v.s. \cref{eq:loss_varphi_dsm}. \cref{fig:abl_tm_dm} shows that, while all three methods eventually converge, AM exhibits \textbf{significantly faster convergence} than DM when trajectory importance reweighting is enabled. Furthermore, we observe that the importance reweighting can inadvertently slow down the convergence due to small effective sample sizes (even after convergence), likely stemming from variance of the estimated RND.

\paragraph{Ablation Study: Memoryless v.s. Non-Memoryless}
We investigate the impact of the noise schedule $\gamma_t$ in \cref{fig:abl_loglin}, and in particular, we show that non-memoryless noise schedule performs better than the memoryless one. We adopt the \textit{modified log-linear schedule} $\gamma_t = \frac{\gamma}{t+\alpha}$ on the $24\times24$ Ising model with $\betah=0.28$, which is memoryless when $\alpha=0$. \cref{fig:abl_loglin} shows that when either hyperparameter is fixed at $1$, the error metrics exhibit a distinct U-shaped curve as a function of the other parameter. Deviating from the optimal regime leads to degradation in sample quality: small $\gamma$ or large $\alpha$ reduces the transition rate magnitude and may hinder exploration, while large $\gamma$ or small $\alpha$ (particularly the memoryless case $\alpha=0$) induces excessive jumps during generation, also degrading performance. Notably, these trends remain consistent across different computational budgets, demonstrating the robustness of the optimal configuration. A similar study for the constant noise schedule $\gamma_t\equiv\gamma$ is provided in \cref{app:exp_further_res}.

\section{Conclusion and Future Work}
\label{sec:conc}
In this work, we introduced a unified framework for discrete SB and SOC, proposing DASBS  as an optimal-policy fixed-point iteration. By leveraging an additive noise scheme based on group structure, DASBS effectively extends AM to discrete domains.
Several limitations remain: 
our evaluation is currently restricted to synthetic benchmarks, and performance on more complex distributions remains unknown.
Additionally, the first-order nature of DASBS may be costly to implement if computing energy is expensive or not parallelizable.
Future directions include extending the framework to discrete SOC with running costs \citep{domingoenrich2025adjoint}, exploring non-uniform reference dynamics like the Ehrenfest process \citep{winkler2024bridging}, and broadening our theoretical insights of AM to general state (e.g., \citet{park2024stochastic,woo2025energybased,park2025functional}).

\bibliography{reference}
\bibliographystyle{icml2026}

\newpage
\appendix
\crefalias{section}{appendix}
\crefalias{subsection}{appendix}
\crefalias{subsubsection}{appendix}
\onecolumn

\section{Related Works}
\label{app:related_works}

\paragraph{1. Adjoint Matching}

The idea of AM \citep{domingoenrich2025adjoint} can be traced back to the adjoint method in optimal control theory \citep{pontryagin1987mathematical}, and a few of its earlier applications in machine learning \citep{han2016deep,chen2018neural,li2020scalable,domingoenrich2024stochastic}. The core idea is to define the adjoint state $a_t(X;u)$ associated with a control $u$ and trajectory $X=(X_t)_{t\in[0,1]}$ as $\nabla_{X_t}J_t(X)$, where $J_t(X)$ the cost-to-go along the trajectory starting from $X_t$. Then, one can establish the dynamics of the adjoint state as an ODE backward in time starting from $a_1$ being the gradient of the terminal cost. Furthermore, the gradient of the full cost-to-go at the initial time $t=0$ with respect to the control parameter can be written as an integral associated with the adjoint state, thus leading to a matching objective using the stop-gradient operator. It is shown in \citet{domingoenrich2025adjoint} that the optimal control is the unique fixed-point of the learning objective.
AM-based methods have been applied to multiple tasks including continuous neural sampler training \citep{havens2025adjoint,liu2025adjoint,choi2025nonequilibrium,blessing2025trust}, fine-tuning diffusion/flow-based models \citep{blessing2025trust,liu2025value,domingoenrich2025a}, and transition path sampling \citep{pidstrigach2025conditioning,howard2025control}, showing competitive performance.

While AM is popularly applied in continuous state spaces, its generalization to discrete state spaces remains unexplored. Recently, \citet{so2026discrete} made a first attempt by defining the discrete adjoint state as an estimator of the exponential of the value difference, and establishing its connection with the optimal transition rate, leading to a matching objective whose unique fixed-point is the optimal transition rate. However, their training loss requires row-out from intermediate states given the samples in the buffer, which is complicated.

\paragraph{2. Discrete Diffusion Neural Samplers}

The study of continuous neural samplers has inspired the construction of similar algorithms for learning discrete distributions.
For uniform-based discrete diffusion, \citet{holderrieth2025leaps} employed the discrete Jarzynski equality and learns the transport via locally equivariant neural networks, \citet{ou2025discrete} followed and refined this line of study by proposing a different loss formulation, and \citet{kholkin2025sampling} proposed a neural sampler based on target concrete score identity.
For mask-based discrete diffusion, \citet{zhu2025mdns} formulated the sampling problem as an SOC problem similar to \citet{zhang2022path}, and proposed samplers based on masked and uniform discrete diffusion.
Finally, discrete diffusion neural samplers are also used for solving combinatorial optimization problems via sampling from a low-temperature target distribution, e.g., \citet{sanokowski2023variational,sanokowski2024a,sanokowski2025scalable,ou2025discrete,guo2026proximal}.

\section{Theoretical Derivation}
\label{app:theory}

\subsection{Kolmogorov Equations for Continuous-time Markov Chains}
\label{app:theory_kol_eq}
\begin{proposition}
    Suppose $p_{t|s}(y|x)$, $0\le s<t\le 1$ are the transition kernels of a CTMC with transition rate $r$. Then the following equations hold:
    \begin{itemize}
        \item \textbf{Kolmogorov forward equation (KFE)}:
        \begin{align}\label{eq:kfe}\tag{KFE}
                \partial_tp_{t|s}(y|x) &= \sum_{z}r_t(y, z) p_{t|s}(z|x),
        \end{align}
        \item \textbf{Kolmogorov backward equation (KBE)}:
        \begin{align}\label{eq:kbe}\tag{KBE}
            \partial_sp_{t|s}(y|x) = - \sum_{z}  r_s(z, x)p_{t|s}(y|z).
        \end{align}
    \end{itemize}
    \label{prop:kol_eq}
\end{proposition}

These are standard results in the theory of CTMC and can be proved by leveraging the definition of the transition rate.

\subsection{Schrödinger Bridge: Proof of \cref{thm:sb_sol}}
\label{app:theory_sb}

We refer readers to \citet{leonard2014a} for a comprehensive review of the SB theory. For completeness, we provide a self-contained proof here.
\begin{lemma}
    There exists some non-negative functions $\phi,\varphi_1$ such that the optimal path measure $p^\star$ satisfies
    \begin{align}
        p^\star_{(0,1)|0,1}&=p^r_{(0,1)|0,1},
        \label{eq:sb_sol_p_star_cond}\\
        p^\star_{0,1}(x,y)&=p^r_{0,1}(x,y)\phi(x)\varphi_1(y),\quad\text{s.t.}~p^\star_0=\mu,~p^\star_1=\nu.
        \label{eq:sb_sol_p_star_marg}
    \end{align}
    In other words, $p^\star(X_{[0,1]})=p^r(X_{[0,1]})\phi(X_0)\varphi_1(X_1)$.
    \label{lem:sb_path_measure}
\end{lemma}
\begin{proof}
By the chain rule of KL divergence, we have
\begin{align*}
    \KL(p^u\|p^r)&=\KL(p^u_{0,1}\|p^r_{0,1})+\E_{p^u_{0,1}(x_0,x_1)}\KL(p^u_{(0,1)|0,1}(\cdot|x_0,x_1)\|p^r_{(0,1)|0,1}(\cdot|x_0,x_1)).
\end{align*}
Therefore, the optimal $p^u_{(0,1)|0,1}$ is $p^r_{(0,1)|0,1}$, and we only need to solve the following \textit{static} SB problem:
\begin{align*}
    \min_{p_{0,1}}~&\KL(p_{0,1}\|p^r_{0,1})\\
    \text{s.t.}~&p_0=\mu,~p_1=\nu\iff\sum_yp_{0,1}(x,y)=\mu(x),~\sum_xp_{0,1}(x,y)=\nu(y).
\end{align*}

Let the Lagrangian multiplier functions for the above constraints be $\alpha(x)$ and $\beta(y)$, respectively. The Lagrangian function is
\begin{align*}
    L(p_{0,1},\alpha,\beta)&=\sum_{x,y}p_{0,1}(x,y)\log\frac{p_{0,1}(x,y)}{p^r_{0,1}(x,y)}\\
    &+\sum_x\alpha(x)\ro{\sum_yp_{0,1}(x,y)-\mu(x)}+\sum_y\beta(y)\ro{\sum_xp_{0,1}(x,y)-\nu(y)}.
\end{align*}
Set the partial derivative with respect to $p_{0,1}(x,y)$ to zero, we have
\begin{align*}
    \log\frac{p_{0,1}(x,y)}{p^r_{0,1}(x,y)}+1+\alpha(x)+\beta(y)=0 \implies p^\star_{0,1}(x,y)=p^r_{0,1}(x,y)\e^{-1-\alpha(x)-\beta(y)}.
\end{align*}
Therefore, by defining $\phi(x):=\e^{-1-\alpha(x)}$ and $\varphi_1(y):=\e^{-\beta(y)}$, we complete the proof. Note that $\phi$ and $\varphi_1$ are defined up to a constant scaling factor.
\end{proof}

\begin{lemma}
    Define $\varphih_0:=\phi p^r_0$, and define the SB potentials $\varphi_t$ and $\varphih_t$ through the following relation:
    \begin{align}\label{eq:def_sb_potentials}
        \varphi_t(x) = \sum_y p^r_{1|t}(y|x) \varphi_1(y), \quad \varphih_t(x) = \sum_y p^r_{t|0}(x|y) \varphih_0(y).
    \end{align}
    Then,
    \begin{align}
        \partial_t \varphi_t(x)&=- \sum_y \varphi_t(y)r_t(y,x),
        \label{eq:hopf-cole-varphi}
        \\
        \partial_t \varphih_t(x)&=\sum_y \varphih_t(y)r_t(x,y),
        \label{eq:hopf-cole-hatvarphi}
    \end{align}
    and furthermore, \cref{eq:varphi_kbe,eq:hatvarphi_kfe} hold.
\end{lemma}

\begin{proof}
First, we prove \cref{eq:hopf-cole-varphi,eq:hopf-cole-hatvarphi}:
\begin{align*}
    \partial_t \varphi_t(x)&=\sum_y \partial_t p^r_{1|t}(y|x)\varphi_1(y)\underset{\cref{eq:kbe}}{=}-\sum_y\sum_z r_t(z,x)p^r_{1|t}(y|z)\varphi_1(y)=-\sum_z r_t(z,x)\varphi_t(z),\\
    \partial_t \varphih_t(x)&=\sum_y \partial_t p^r_{t|0}(x|y)\varphih_0(y)\underset{\cref{eq:kfe}}{=}\sum_y\sum_z r_t(x,z)p^r_{t|0}(z|y)\varphih_0(y)=\sum_z r_t(x,z)\varphih_t(z).
\end{align*}

Next, we prove \cref{eq:varphi_kbe,eq:hatvarphi_kfe} by verifying that the partial derivatives to $t$ in \cref{eq:varphi_kbe} and to $s$ in \cref{eq:hatvarphi_kfe} are zero:
\begin{align*}
    \sum_y\partial_t(p^r_{t|s}(y|x)\varphi_t(y))&=\sum_y\sum_z r_t(y,z)p^r_{t|s}(z|x)\varphi_t(y)-\sum_y\sum_z p^r_{t|s}(y|x)r_t(z,y)\varphi_t(z)=0,\\
    \sum_y\partial_s(p^r_{t|s}(x|y)\varphih_s(y))&=-\sum_y\sum_z r_s(z,y)p^r_{t|s}(x|z)\varphih_s(y)+\sum_y\sum_z p^r_{t|s}(x|y)r_s(y,z)\varphih_s(z)=0.
\end{align*}

\end{proof}

\textit{Proof of \cref{eq:sb_relation}.}
We first study the marginal distribution at an arbitrary time point $t\in[0,1]$:
\begin{align*}
    p^\star_t(x)&=\sum_{x_0,x_1} p^\star_{0,1}(x_0,x_1)p^\star_{t|0,1}(x|x_0,x_1)
    \underset{\text{\cref{eq:sb_sol_p_star_cond,eq:sb_sol_p_star_marg}}}{=}\sum_{x_0,x_1}p^r_{0,1}(x_0,x_1)\phi(x_0)\varphi_1(x_1)p^r_{t|0,1}(x|x_0,x_1)\\
    &=\sum_{x_0} p^r_{0}(x_0)p^r_{t|0}(x|x_0)\phi(x_0)\sum_{x_1} p^r_{1|t}(x_1|x)\varphi_1(x_1)\underset{\cref{eq:def_sb_potentials}}{=}\varphih_t(x)\varphi_t(x).
\end{align*}

Next, we study the joint distribution at two arbitrary time points $0\le s<t\le 1$:
\begin{align*}
    p^\star_{s,t}(x,y)&=\sum_{x_0,x_1} p^\star_{0,1}(x_0,x_1)p^\star_{s,t|0,1}(x,y|x_0,x_1)\\
    &\underset{\text{\cref{eq:sb_sol_p_star_cond,eq:sb_sol_p_star_marg}}}{=}\sum_{x_0,x_1}p^r_{0,1}(x_0,x_1)\phi(x_0)\varphi_1(x_1)p^r_{s,t|0,1}(x,y|x_0,x_1)\\
    &=\sum_{x_0} p^r_{0,s,t,1}(x_0,x,y,x_1)\phi(x_0)\varphi_1(x_1)\\
    &=\sum_{x_0} p^r_{0}(x_0)p^r_{s|0}(x|x_0)\phi(x_0)\sum_{x_1} p^r_{t|s}(y|x)p^r_{1|t}(x_1|y)\varphi_1(x_1)\\
    &\underset{\cref{eq:def_sb_potentials}}{=}\varphih_s(x)p^r_{t|s}(y|x)\varphi_t(y).
\end{align*}

Therefore, the second and third equalities in \cref{eq:sb_relation} follow immediately.

\textit{Proof of \cref{eq:sb_opt_path_measure}.} 
This is obvious from \cref{lem:sb_path_measure,eq:sb_relation}:
\begin{align*}
    p^\star(X_{[0,1]})&=p^r(X_{[0,1]})\phi(X_0)\varphi_1(X_1)=p^r(X_{[0,1]})\frac{\varphih_0(X_0)}{\mu(X_0)}\varphi_1(X_1)=p^r(X_{[0,1]})\frac{\varphih_0(X_0)}{\mu(X_0)}\frac{\nu(X_1)}{\varphih_1(X_1)}.
\end{align*}
\hfill\qedsymbol

\textit{Proof of \cref{eq:u_star}.} From \cref{eq:sb_relation}, for $y\ne x$,
\begin{align*}
    u_t^\star(y,x)&=\lim_{h\to0}\frac{p^\star_{t+h|t}(y|x)}{h}=\lim_{h\to0}\frac{p^r_{t+h|t}(y|x)}{h}\frac{\varphi_{t+h}(y)}{\varphi_t(x)}=\frac{\varphi_t(y)}{\varphi_t(x)}\lim_{h\to0}\frac{p^r_{t+h|t}(y|x)}{h}=\frac{\varphi_t(y)}{\varphi_t(x)}r_t(y,x).
\end{align*}
\hfill\qedsymbol

\subsection{Stochastic Optimal Control: Proof of \cref{eq:soc_opt_path_measure}}
\label{app:theory_soc}

\begin{proof}
Using the chain rule of KL divergence and as $p^u_0=p^r_0=\mu$, we have
\begin{align*}
    \KL(p^u\|p^r)+\E_{X\sim p^u}g(X_1)&=\E_{\mu(X_0)}\sq{\E_{p^u(X_{(0,1]}|X_0)}\log\frac{p^u(X_{(0,1]}|X_0)}{p^r(X_{(0,1]}|X_0)\e^{-g(X_1)}}}.
\end{align*}
Therefore, for any $X_0$, the optimal $p^u(X_{(0,1]}|X_0)$ is
\begin{align*}
    p^\star(X_{(0,1]}|X_0)=\frac{1}{Z(X_0)}p^r(X_{(0,1]}|X_0)\e^{-g(X_1)},
\end{align*}
where
\begin{align*}
    Z(X_0)=\E_{p^r(X_{(0,1]}|X_0)\e^{-g(X_1)}}=\E_{p^r_{1|0}(y|X_0)}\e^{-g(y)}.
\end{align*}
Finally, combining this with $p^\star_0=p^r_0=\mu$ completes the proof.
\end{proof}

\subsection{Proof of \cref{thm:sb_soc_relation}}
\label{app:theory_soc_sb}
\begin{proof}
See \cref{app:theory_sb,app:theory_soc} for the proofs of \cref{eq:sb_opt_path_measure,eq:soc_opt_path_measure}, respectively. To conclude the proof, it suffices to show the equivalence when $g\gets\log\frac{\varphih_1}{\nu}$:
\begin{align*}
    Z(x)&\underset{\cref{eq:soc_opt_path_measure}}{=}\sum_y\frac{p^r_{1|0}(y|x)}{\varphih_1(y)}\nu(y)\underset{\cref{eq:sb_relation}}{=}\sum_y\frac{p^\star_{0|1}(x|y)}{\varphih_0(x)}\nu(y)=\sum_y\frac{p^\star_{0,1}(x,y)}{\varphih_0(x)}=\frac{\mu(x)}{\varphih_0(x)}.
\end{align*}
\end{proof}

\begin{corollary}
    \label{thm:soc_to_sb}
    The problem \cref{eq:soc} with terminal cost $g$ is equivalent to the problem \cref{eq:sb_prob} with terminal distribution
    \begin{align*}
        \nu(x)=\e^{-g(x)}\sum_yp^r_{1|0}(x|y)\frac{\mu(y)}{Z(y)},\qquad\text{where}~Z(y)=\E_{p^r_{1|0}(\cdot|y)}\e^{-g}.
    \end{align*}
\end{corollary}

\begin{proof}
It suffices to express $\nu$ by $g$ by comparing \cref{eq:sb_opt_path_measure,eq:soc_opt_path_measure}. We have $\mu=\varphih_0Z$ and $\nu=\varphih_1\e^{-g}$. From \cref{eq:hatvarphi_kfe}, $\varphih_1(x)=\sum_yp^r_{1|0}(x|y)\varphih_0(y)$. Combining all these three equations completes the proof.
\end{proof}

\subsection{Uniform Reference Dynamics}
\label{app:theory_p_ref}
We write $r_t(y,x)=\gamma_tr(y,x)$ where
\begin{align*}
    r(y,x)=\begin{cases}
        \frac{1}{N},&\text{if }\dh(y,x)=1,\\
        -D\ro{1-\frac1N},&\text{if }y=x,\\
        0,&\text{if otherwise.}
    \end{cases}
\end{align*}
By verifying the detailed balance condition, one is easy to see that this transition rate keeps $\punif=\unif(\cX)$ invariant:
\begin{align*}
    \punif(x)r_t(y,x)=\punif(y)r_t(x,y)=\frac{\gamma_t}{N^{D+1}}1_{\dh(y,x)=1},~\forall x\ne y.
\end{align*}

\begin{proposition}
    Define $\gammab_{s,t}:=\int_s^t\gamma_u\d u$. Then for $x,y\in\cX$ and $0\le s<t\le1$, the reference transition probability is
    \begin{equation}
        p^r_{t|s}(y|x)=A(s,t)^{\dh(x,y)}B(s,t)^{D-\dh(x,y)},~
        \text{where}~A(s,t):=\frac{1-\e^{-\gammab_{s,t}}}{N},~B(s,t):=\frac{1+(N-1)\e^{-\gammab_{s,t}}}{N}.
    \label{eq:p_r_trans_prob}
    \end{equation}
    \label{prop:p_ref_trans_prob}
\end{proposition}

\begin{remark}
    As a corollary, $p^r_{1|t}(y|x)=q_t(y-x)$ where $q_t(\varepsilon)=A(t,1)^{\dh(\varepsilon,0)}B(t,1)^{D-\dh(\varepsilon,0)}$.
\end{remark}

\begin{proof}
Note that each dimension evolves independently with transition rate matrix $\gamma_t\bR$ where $\bR\in\R^{N\times N}$ has off-diagonal entries $\frac1N$ and diagonal entries $\frac1N-1$, i.e., $\bR=\frac1N\one\one\tp-\bI$.
For $x\in[N]$, the vector $\bp^r_{t|s}(\cdot|x):=(p^r_{t|s}(y|x))_{y\in[N]}$ satisfies the Kolmogorov forward equation $\partial_t\bp^r_{t|s}(\cdot|x)=\gamma_t\bR\bp^r_{t|s}(\cdot|x)$ with initial condition $\bp^r_{s|s}(\cdot|x)=\be^x$, i.e., the one-hot vector with $1$ at the $x$-th entry. Therefore, we have $\bp^r_{t|s}(\cdot|x)=\e^{\gammab_{s,t}\bR}\be^x$.

The matrix exponential can be computed as follows:
\begin{align*}
    &\e^{s\one\one\tp}=\sum_{k=0}^\infty\frac{s^k}{k!}(\one\one\tp)^k=\bI+\sum_{k=1}^\infty\frac{s^k}{k!}N^{k-1}\one\one\tp=\bI+\frac{\e^{sN}-1}{N}\one\one\tp,\\
    \implies&\e^{\lambda\bR}=\e^{\frac\lambda N\one\one\tp-\lambda\bI}=\e^{-\lambda}\e^{\frac\lambda N\one\one\tp}=\e^{-\lambda}\ro{\bI+\frac{\e^\lambda-1}{N}\one\one\tp}=\e^{-\lambda}\bI+\frac{1-\e^{-\lambda}}{N}\one\one\tp.
\end{align*}

Therefore, the per-dimension transition probability for $x,y\in[N]$ is
\begin{align}
    p^r_{t|s}(y|x)&=\begin{cases}
        \frac{1-\e^{-\gammab_{s,t}}}{N}=:A(s,t) & \text{if }y\ne x,\\
        \frac{1+(N-1)\e^{-\gammab_{s,t}}}{N}=:B(s,t) & \text{if }y=x,
    \end{cases}
    \label{eq:p_r_trans_prob_1d}
\end{align}
which implies the full transition probability for $x,y\in\cX$ is \cref{eq:p_r_trans_prob}.

\end{proof}

\begin{remark}
    Note that when $\gammab_{0,1}=\infty$, we have $p^r_{1|0}(y|x)=\frac{1}{N^D}$ for all $x,y\in\cX$, i.e., the reference dynamics is memoryless.
    Under the assumption that $p^r_0=\punif$, $p^r_t=\punif$ and thus $p^r_{t|1}(\cdot|y)$ can be implemented as follows: for each entry of $y$, independently, with probability $1-\e^{-\gammab_{t,1}}$, replace it with a random state from $\unif[N]$.
\end{remark}

\begin{proposition}
    Under the reference dynamics $p^r$, each dimension of $p^r_{t|0,1}(x|x_0,x_1)$ for $x,x_0,x_1\in\cX$ is independent, and the per-dimensional distribution is given by
    \begin{align}
        \begin{split}
        p^r_{t|0,1}(x|x_0,x_1)&=
        \begin{cases}
            \text{if }x_0\ne x_1:&
            \begin{cases}
                \frac{A(0,t)A(t,1)}{A(0,1)}, & \text{if }x\notin\{x_0,x_1\},\\
                \frac{B(0,t)A(t,1)}{A(0,1)}, & \text{if }x=x_0,\\
                \frac{A(0,t)B(t,1)}{A(0,1)}, & \text{if }x=x_1,
            \end{cases}
            \\
            \text{if }x_0=x_1:&
            \begin{cases}
                \frac{A(0,t)A(t,1)}{B(0,1)}, & \text{if }x\ne x_0=x_1,\\
                \frac{B(0,t)B(t,1)}{B(0,1)}, & \text{if }x=x_0=x_1,
            \end{cases}
        \end{cases}
        \end{split}
        \label{eq:p_r_bridge_1d}
    \end{align}
    for $x,x_0,x_1\in[N]$.
    \label{prop:p_ref_bridge}
\end{proposition}

\begin{proof}    
Due to the Markov property of $p^r$, we have
\begin{align*}
    p^r_{t|0,1}(x|x_0,x_1)&=\frac{p^r_{0,t,1}(x_0,x,x_1)}{p^r_{0,1}(x_0,x_1)}=\frac{p^r_{t|0}(x|x_0)p^r_{1|t}(x_1|x)}{p^r_{1|0}(x_1|x_0)}.
\end{align*}
Note that again, each dimension is independent under $p^r$. Using the per-dimension transition probability \cref{eq:p_r_trans_prob_1d}, one can easily obtain the desired result.
\end{proof}

\paragraph{Choice of Noise Schedule $\gamma_t$}
We consider the following noise schedules:
\begin{itemize}
    \item Constant schedule: $\gamma_t\equiv\gamma>0$, $\gammab_{s,t}=\gamma(t-s)$.
    \item Modified log-linear schedule: $\gamma_t=\frac{\gamma}{t+\alpha}$ for some $\gamma,\alpha>0$, $\gammab_{s,t}=\gamma\log\frac{t+\alpha}{s+\alpha}$. A larger $\alpha$ means a stronger memory effect, and $\alpha=0$ recovers the memoryless case.
\end{itemize}
The ablation studies of the noise schedules can be found at \cref{fig:abl_const,fig:abl_loglin}.

\subsection{Inference via $\tau$-leaping}
\label{app:theory_tau_leap}

\begin{proposition}
    Using $\tau$-leaping to discretize the controlled CTMC, the transition probability for any dimension $d$ is approximated as
    \begin{align*}
        \Pr(X_{t+h}^d=n|X_t=x)\approx&\begin{cases}
            \frac{\gammab_{t,t+h}}{N}\varPhi_t(x)_{d,n}, & \text{if}~n\ne x^d,\\
            1-\frac{\gammab_{t,t+h}}{N}\sum_{n'\ne x^d}\varPhi_t(x)_{d,n'}, & \text{if}~n=x^d.
        \end{cases}
    \end{align*}
    The approximate transition probability can be computed as
    \begin{align}
        p^u_{t+h|t}(y|x)&\approx\prod_{d=1}^D\Pr(X_{t+h}^d=y^d|X_t=x)\nonumber\\
        &=\prod_{d=1}^D\ro{\frac{\gammab_{t,t+h}}{N}\varPhi_t(x)_{d,y^d}}^{1_{y^d\ne x^d}}\ro{1-\frac{\gammab_{t,t+h}}{N}\sum_{n'\ne x^d}\varPhi_t(x)_{d,n'}}^{1_{y^d=x^d}}.\label{eq:p_u_trans_prob}
    \end{align}
    \label{prop:p_u_trans_prob}
\end{proposition}

We summarize the sampling procedure in \cref{alg:dasbs_sample}.
Note that for uniform discrete diffusion models, other samplers such as uniformization \citep{chen2025convergence,ren2025how} and higher-order $\tau$-leaping \citep{ren2025fast} can also be applied, which we leave as future works.

\begin{algorithm}[ht]
\caption{Sampling of discrete adjoint Schrödinger bridge sampler (DASBS) via $\tau$-leaping}
\label{alg:dasbs_sample}
\begin{algorithmic}[1]
\REQUIRE Model $\varPhi$, initial distribution $\mu$, time discretization $0=t_0<t_1<...<t_M=1$.
\STATE Sample initial state $x_0\sim\mu$.
\FOR{$i=0$ \textbf{to} $M-1$}
    \STATE Query $\varPhi_{t_i}(x_i)\in\R_+^{D\times N}$ and compute $\gammab_{t_i,t_{i+1}}$.
    \STATE For each $d\in[D]$ (in parallel), independently sample
    \begin{align*}
        \Pr(x_{i+1}^d=n)&=\begin{cases}
            \frac{\gammab_{t_i,t_{i+1}}}{N}\varPhi_{t_i}(x_i)_{d,n}, & \text{if } n\ne x_i^d,\\
            1-\frac{\gammab_{t_i,t_{i+1}}}{N}\sum_{n'\ne x_i^d}\varPhi_{t_i}(x_i)_{d,n'}, & \text{if } n=x_i^d.
        \end{cases}
    \end{align*}
\ENDFOR
\OUTPUT $x_M$ as the generated sample.
\end{algorithmic}
\end{algorithm}

\subsection{Target Matching Loss}
\label{app:theory_varphih_star_tsm}
\textit{Proof of \cref{eq:varphih_star_tsm}.}
\begin{align*}
    \varphih_1(z)
    &\underset{\cref{eq:hatvarphi_kfe}}{=}
    \sum_{x\in\Z_N^D}p^r_{1|t}(z|x)\varphih_t(x)
    \underset{\cref{eq:additive_noise}}{=}
    \sum_{x\in\Z_N^D}p^r_{1|t}(z+\Delta|x+\Delta)\varphih_t(x)\\
    &\underset{x'\gets x+\Delta}{=}\sum_{x'\in\Z_N^D}p^r_{1|t}(z+\Delta|x')\varphih_t(x'-\Delta)
    \underset{\Delta\gets y-z}{=}\sum_{x'\in\Z_N^D}p^r_{1|t}(y|x')\varphih_t(x'-y+z),\\
    \implies\frac{\varphih_1(z)}{\varphih_1(y)}
    &=\sum_xp^r_{1|t}(y|x)\frac{\varphih_t(x-y+z)}{\varphih_1(y)}\underset{\cref{eq:sb_relation}}{=}
    \sum_xp^\star_{t|1}(x|y)\frac{\varphih_t(x-y+z)}{\varphih_t(x)}.
\end{align*}
\hfill\qedsymbol

\subsection{Denoising Matching Loss}
\label{app:theory_dm}
We first prove the denoising matching characterization of the controller: 
\begin{align}
    \frac{\varphi_t(y)}{\varphi_t(x)}
    &\underset{\cref{eq:varphi_kbe}}{=}\sum_z p^r_{1|t}(z|y)\frac{\varphi_1(z)}{\varphi_t(x)} = \sum_z \frac{p^r_{1|t}(z|y)}{p^r_{1|t}(z|x)}\frac{p^r_{1|t}(z|x)\varphi_1(z)}{\varphi_t(x)}
    \underset{\cref{eq:sb_relation}}{=}\sum_z \frac{p^r_{1|t}(z|y)}{p^r_{1|t}(z|x)} p^\star_{1|t}(z|x),
    \label{eq:varphi_ratio_dm}\\
    \implies\varPhi^\star&=\argmin_\varPhi\E_t\E_{p^\star_{t,1}(x,x_1)}\sum_{d=1}^D\sum_{n\ne x^d}D_f\ro{\frac{p^r_{1|t}(x_1|x^{d\gets n})}{p^r_{1|t}(x_1|x)}\middle\|\varPhi_t(x)_{d,n}}.
    \label{eq:varphi_star_dsm}
\end{align}

Thus, the denoising matching loss for the controller reads
\begin{align}
    \varPhi^{(k)} &:= \argmin_\varPhi
    \E_tw_t\E_{\substack{p^{\sg(r\varPhi)}_{0,1}(x_0,x_1)\\ p^r_{t|0,1}(x|x_0,x_1)}}\sum_{d=1}^D\sum_{n\ne x^d} D_f\ro{\colordm{\frac{p^r_{1|t}(x_1|x^{d\gets n})}{p^r_{1|t}(x_1|x)}}\middle\|\varPhi_t(x)_{d,n}}.\label{eq:loss_varphi_dsm}\tag{ctrl-DM}
\end{align}

Moreover, the variational characterization of $\varPhih^\star$ in the style of denoising matching is as follows:
\begin{align*}
    \varPhih^\star=\argmin_{\varPhih}\E_t\E_{p^\star_{t,1}(x,x_1)}\sum_{d=1}^D\sum_{n\ne x^d}D_f\ro{\frac{p^r_{1|t}(x_1^{d\gets n}|x)}{p^r_{1|t}(x_1|x)}\middle\|\varPhih(x_1)_{d,n}}.
\end{align*}

\subsection{Trajectory Importance Reweighting}
\label{app:theory_imp_samp}
By \textbf{trajectory importance reweighting}, we refer to the practice when we use the RND $\frac{p^\star(x_{[0,1]})}{p^u(x_{[0,1]})}$ in computing the losses \cref{eq:loss_varphi_tsm,eq:loss_varphi_dsm,eq:loss_hatvarphi_tsm,eq:loss_hatvarphi_dsm}. For instance, \cref{eq:loss_hatvarphi_tsm} becomes the following form:
\begin{align*}
    \varPhi^{(k)} &:= \argmin_\varPhi~
    \E_tw_t\boxed{\E_{p^{\sg(r\varPhi)}(x_{[0,1]})}\frac{p^\star(x_{[0,1]})}{p^{\sg(r\varPhi)}(x_{[0,1]})}}\E_{p^r_{t|0,1}(x|x_0,x_1)}\sum_{d=1}^D\sum_{n\ne x^d}D_f\ro{\colortm{\frac{\varphi_1(x_1^{d\gets x_1^d+n-x^d})}{\varphi_1(x_1)}}\middle\|\varPhi_t(x)_{d,n}}.
\end{align*}
In practice, for stability, after obtaining the log RND $\log\frac{p^\star(x_{[0,1]})}{p^{\sg(r\varPhi)}(x_{[0,1]})}$ up to an additive constant over a batch of trajectories, one typically apply \textit{softmax} to normalize the sum of weights to one.

Leveraging \cref{prop:p_u_trans_prob,prop:p_ref_trans_prob}, we can approximately compute $\log\frac{p^u(x_{[0,1]})}{p^r(x_{[0,1]})}$. With the approximation in the following \cref{lem:log_p_star_p_r}, we can approximate the log weights for importance sampling, $\log\frac{p^\star(x_{[0,1]})}{p^u(x_{[0,1]})}$.

\begin{lemma}
    When the reference dynamics is memoryless, $\log\frac{p^\star(x_{[0,1]})}{p^r(x_{[0,1]})}=\log\frac{\nu(x_1)}{p^r_1(x_1)}+\const$; otherwise, the log RND can be computed and approximated as follows:
    \begin{align}
        &\log\frac{p^\star(x_{[0,1]})}{p^r(x_{[0,1]})}=\int_0^1\sum_{y\ne x_t}\ro{1-\frac{\varphi_t(y)}{\varphi_t(x_t)}}r_t(y,x_t)\d t+\sum_{t:x_{t_-}\ne x_t}\log\frac{\varphi_t(x_t)}{\varphi_t(x_{t_-})}
        \label{eq:log_p_star_p_r_exact}\\
        &\approx\sum_{i=0}^{M-1}\sq{\sum_{d=1}^D\sum_{n\ne x_{t_i}^d}\frac{\gammab_{t_i,t_{i+1}}}{N}\ro{1-\varPhi_{t_i}(x_{t_i})_{d,n}}+\sum_{d:x_{t_i}^d\ne x_{t_{i+1}}^d}\log\varPhi_{t_i}(x_{t_i})_{d,x_{t_{i+1}}^d}},
        \label{eq:log_p_star_p_r_approx}
    \end{align}
    where $0=t_0<...<t_M=1$ are the discretized time points.
    \label{lem:log_p_star_p_r}
\end{lemma}

\begin{proof}
From \cref{eq:sb_opt_path_measure}, one can find $\log\frac{p^\star(x_{[0,1]})}{p^r(x_{[0,1]})}=\log\frac{\varphi_1(x_1)}{\varphi_0(x_0)}$. When memoryless, following the argument in \cref{sec:theory}, $\varphi_0=\const$, $\log\varphi_1=\log\frac{\nu}{\varphih_1}=\log\frac{\nu}{p^r_1}+\const$, which yields the desired result.

For general cases, inspired by the discussion in \citet[App. D.4]{liu2025adjoint}: \cref{eq:hopf-cole-varphi} implies
\begin{align*}
    \partial_t\varphi_t(x)=\sum_{y\ne x}(\varphi_t(x)-\varphi_t(y))r_t(y,x)\implies\partial_t\log\varphi_t(x)=\sum_{y\ne x}\ro{1-\frac{\varphi_t(y)}{\varphi_t(x)}}r_t(y,x).
\end{align*}
Hence, we have
\begin{align*}
    \log\frac{\varphi_1(x_1)}{\varphi_0(x_0)}&=\int_0^1\partial_t\log\varphi_t(x_t)\d t+\sum_{t:x_{t_-}\ne x_t}\log\frac{\varphi_t(x_t)}{\varphi_t(x_{t_-})}\\
    &=\int_0^1\sum_{y\ne x_t}\ro{1-\frac{\varphi_t(y)}{\varphi_t(x_t)}}r_t(y,x_t)\d t+\sum_{t:x_{t_-}\ne x_t}\log\frac{\varphi_t(x_t)}{\varphi_t(x_{t_-})}.
\end{align*}

Under our setting, the summation in the first term can be easily calculated:
\begin{align*}
    \sum_{y\ne x_t}\ro{1-\frac{\varphi_t(y)}{\varphi_t(x_t)}}r_t(y,x_t)&=\sum_{d=1}^D\sum_{n\ne x_t^d}\ro{1-\frac{\varphi_t(x_t^{d\gets n})}{\varphi_t(x_t)}}r_t(x_t^{d\gets n},x_t)=\sum_{d=1}^D\sum_{n\ne x_t^d}\ro{1-\varPhi^\star_t(x_t)_{d,n}}\frac{\gamma_t}{N}.
\end{align*}
For the second term, in theory, $x_{t_-}$ and $x_t$ differ by one dimension only, but during discretization, multiple dimensions may change simultaneously. A heuristic approximation is to decompose it into multiple single-dimension changes:
\begin{align*}
    \log\frac{\varphi_t(x_t)}{\varphi_t(x_{t_-})}&\approx\sum_{d:x_{t_-}^d\ne x_t^d}\log\frac{\varphi_t(x_{t_-}^{d\gets x_t^d})}{\varphi_t(x_{t_-})}=\sum_{d:x_{t_-}^d\ne x_t^d}\log\varPhi^\star_t(x_{t_-})_{d,x_t^d}.
\end{align*}
We can thus summarize the approximate calculation of $\frac{p^\star(x_{[0,1]})}{p^r(x_{[0,1]})}=\frac{\varphi_1(x_1)}{\varphi_0(x_0)}$ on the time-discretized trajectory $(x_{t_0},x_{t_1},...,x_{t_M})$ as follows:
\begin{align*}
    \log\frac{\varphi_1(x_1)}{\varphi_0(x_0)}&\approx\sum_{i=0}^{M-1}\sq{\sum_{d=1}^D\sum_{n\ne x_{t_i}^d}\frac{\gammab_{t_i,t_{i+1}}}{N}\ro{1-\varPhi^\star_{t_i}(x_{t_i})_{d,n}}+\sum_{d:x_{t_i}^d\ne x_{t_{i+1}}^d}\log\varPhi^\star_{t_i}(x_{t_i})_{d,x_{t_{i+1}}^d}}.
\end{align*}
Finally, as the ground-truth $\varPhi^\star$ is unavailable, we use $\varPhi$ to approximate its value. This concludes the proof.

\end{proof}

\begin{remark}
    Unlike the method proposed in \citet[App. D.4]{liu2025adjoint}, here we don't need to train $\frac{\varphih_t(y)}{\varphih_t(x)}$ along the trajectory. This is because here we derive through $\frac{\varphi_1(x_1)}{\varphi_0(x_0)}$ instead of $\frac{\varphih_1(x_1)}{\varphih_0(x_0)}$.
\end{remark}

\begin{remark}
    In the case of \textit{masked} diffusion, the path measure $p^u$ can be \textit{exactly} sampled and the log RND on a give trajectory can be computed \textit{precisely}; however, here, the estimated $\log\frac{p^\star(x_{[0,1]})}{p^u(x_{[0,1]})}$ involves \textit{time discretization error}, and furthermore,  \textit{learning error} if we use \cref{eq:log_p_star_p_r_approx} under non-memoryless cases. This possibly explains the low ESS observed in \cref{fig:abl_tm_dm} even after convergence.
\end{remark}

\subsection{Connection between DASBS and WDCE}
\label{app:theory_wdce}

\begin{proposition}
    Assume $\gammab_{0,1}=\infty$, i.e., the reference path measure $p^r$ is memoryless. Then the denoising loss for training the controller \cref{eq:loss_varphi_dsm} with trajectory importance reweighting, $D_f$ being the generalized KL divergence,\footref{fn:gen_kl} and time-weight $w_t\gets\frac{\gamma_t}{N}$ reduces the denoising cross-entropy (WDCE) loss in UDNS \citep[App. F]{zhu2025mdns}, which is equal to $\KL(p^\star\|p^u)+\const$.
    \label{prop:udns}
\end{proposition}

\begin{proof}
The first equation on \citet[Page 39]{zhu2025mdns} reads (using the notation in this paper)
\begin{align*}
    \KL(p^\star\|p^u)&=\E_{p^{\sg(u)}(\xb_{[0,1]})}\frac{p^\star(\xb_{[0,1]})}{p^{\sg(u)}(\xb_{[0,1]})}\E_t\frac{\gamma_t}{N}\E_{p^\star_{t|1}(x|\xb_1)}\sum_{d=1}^D\sum_{n\ne x^d}D_f\ro{\frac{p^\star_{t|1}(x^{d\gets n}|\xb_1)}{p^\star_{t|1}(x|\xb_1)}\middle\|\varPhi_t(x)_{d,n}}+\const,
\end{align*}
where $D_f$ is the generalized KL divergence. Under the memoryless assumption, we have $p^\star_{t|1}(x|x_1)=p^r_{t|1}(x|x_1)$. Due to the symmetry $p^r_{t|1}(x|x_1)=p^r_{t|1}(x_1|x)$, the equivalence to \cref{eq:loss_varphi_dsm} is obvious. 

\end{proof}

\begin{proposition}
    Consider the mask-augmented state space $\cX=\{1,...,N,\mask\}^D$. Let the reference transition rate be 
    \begin{align}
    r_t(y,x)=\begin{cases}
        \frac{\gamma_t}{N},&\text{if }y=x^{d\gets n},~x^d=\mask,~n\in[N],\\
        -\gamma_t|\{d:x^d=\mask\}|,&\text{if }y=x,\\
        0,&\text{otherwise,}
    \end{cases}
    \label{eq:ref_transition_rate_mask}
    \end{align}
    for some noise schedule $\gamma_\cdot:[0,1]\to\R_+$. Let $\gammab_{s,t}:=\int_s^t\gamma_u\d u$ for $0\le s<t\le1$, and assume $\gammab_{t,1}=\infty$ for any $0\le t<1$. Then, the denoising loss for training the controller \cref{eq:loss_varphi_dsm} with trajectory importance weighting, $D_f$ being the generalized KL divergence,\footref{fn:gen_kl} and time-weight $\frac{\gamma_t}{N}$ reduces to the WDCE loss in MDNS \citep{zhu2025mdns}, which is also equal to $\KL(p^\star\|p^u)+\const$.
    \label{prop:mdns}
\end{proposition}

\begin{proof}
Throughout the proof, we always assume $n\in[N]$ is a non-mask state.
From \citet{zhu2025mdns}, we have the following results:

\begin{align*}
    u^\star_t(x^{d\gets n},x)&=\gamma_t\Pr_{X\sim\nu}(X^d=n|X^\um=x^\um)1_{x^d=\mask}\\
    \varPhi^\star_t(x)_{d,n}=\frac{\varphi_t(x^{d\gets n})}{\varphi_t(x)}&=N\Pr_{X\sim\nu}(X^d=n|X^\um=x^\um)1_{x^d=\mask}\\
    p^r_{t|s}(y|x)&=\prod_{d:x^d=\mask}\ro{\frac{1-\e^{-\gammab_{s,t}}}{N}}^{1_{y^d\ne\mask}}\ro{\e^{-\gammab_{s,t}}}^{1_
    {y^d=\mask}}\cdot\prod_{d:x^d\ne\mask}1_{x^d=y^d},~0\le s<t\le1\\
    \implies\frac{p^r_{1|t}(x_1|x^{d\gets n})}{p^r_{1|t}(x_1|x)}&=N1_{x_1^d=n}1_{x^d=\mask}~(\text{suppose $\forall d$ s.t. $x^d\ne\mask$, $x_1^d=x^d$})\\
\end{align*}

Let $s_\theta:\cX\to\R^{D\times N}_{\ge0}$ be the neural network to learn the conditional distribution in $\nu$, i.e., $s_\theta(x)_{d,n}\approx\Pr_{X\sim\nu}(X^d=n|X^\um=x^\um)$ for $x^d=\mask$, and $s_\theta(x)_{d,n}=1_{x^d=n}$ for $x^d\ne\mask$. We assume $\sum_{n=1}^Ns_\theta(x)_{d,n}=1$ for all $d$.

Therefore, with generalized KL divergence, the loss \cref{eq:loss_varphi_dsm} with trajectory importance reweighting becomes
\begin{align*}
    &\E_tw_t\E_{p^{\sg(\theta)}(x_{[0,1]})}\de{p^\star}{p^{\sg(\theta)}}(x_{[0,1]})\E_{p^r_{t|0,1}(x|x_0,x_1)}\sum_{d=1}^D\sum_{n\ne x^d}D_f(N1_{x_1^d=n}1_{x^d=\mask}\|Ns_\theta(x)_{d,n}1_{x^d=\mask}+1_{x^d=n})\\
    &=\E_tw_t\E_{p^{\sg(\theta)}(x_{[0,1]})}\de{p^\star}{p^{\sg(\theta)}}(x_{[0,1]})\E_{p^r_{t|1}(x|x_1)}\sum_{d:x^d=\mask}\sum_{n=1}^ND_f(N1_{x_1^d=n}\|Ns_\theta(x)_{d,n})\\
    &=\E_tw_t\E_{p^{\sg(\theta)}(x_{[0,1]})}\de{p^\star}{p^{\sg(\theta)}}(x_{[0,1]})\E_{p^r_{t|1}(x|x_1)}\sum_{d:x^d=\mask}\sum_{n=1}^N\ro{N1_{x_1^d=n}\log\frac{1_{x_1^d=n}}{s_\theta(x)_{d,n}}-N1_{x_1^d=n}+Ns_\theta(x)_{d,n}}\\
    &=\E_tw_t\E_{p^{\sg(\theta)}(x_{[0,1]})}\de{p^\star}{p^{\sg(\theta)}}(x_{[0,1]})\E_{p^r_{t|1}(x|x_1)}\sum_{d:x^d=\mask}-N\log s_\theta(x)_{d,x_1^d}+\const.
\end{align*}

Thus, with $t\sim\unif(0,1)$ and time weights $\frac{\gamma_t}{N}$, the loss
\begin{align*}
    \min_\theta\E_{p^{\sg(\theta)}(x_{[0,1]})}\de{p^\star}{p^{\sg(\theta)}}(x_{[0,1]})\E_{t\sim\unif(0,1)}\gamma_t\E_{p^r_{t|1}(x|x_1)}\sum_{d:x^d=\mask}-\log s_\theta(x)_{d,x_1^d},
\end{align*}
which is exactly the same as the WDCE loss in \citet{zhu2025mdns} if we choose the canonical noise schedule in masked diffusion model: $\gamma_t=\frac{1}{t}$.
\end{proof}

\subsection{Convergence of Alternating Update: Proof of \cref{thm:cvg}}
\label{app:theory_cvg}

\begin{definition}
    \label{def:cvg_q_psi}
    For a function $\psi:\cX\to\R_+$, we use $q^\psi$ to denote the path measure of a CTMC $(Y_t)_{t\in[0,1]}$ induced by the \emph{backward} transition rate $\ul_t(y,x)=\frac{\psi_t(y)}{\psi_t(x)}r_t(x,y)$, $y\ne x$ and initialized at $Y_1\sim\nu$, where $\psi_t$ satisfies $\psi_t(x)=\sum_y p^r_{t|0}(x|y)\psi_0(y)$, $\psi_1=\psi$. In other words, 
    \begin{align*}
        \ul_t(y,x)&=\lim_{h\to0}\frac{\Pr(Y_{t-h}=y|Y_t=x)-1_{x=y}}{h}.
    \end{align*}
\end{definition}

\begin{remark}
    By the Bayes formula, one can obtain its equivalent forward transition rate \citep{kelly2011reversibility}: for $y\ne x$,
    \begin{align*}
        \ur_t(y,x)&=\lim_{h\to0}\frac{1}{h}\Pr(Y_{t+h}=y|Y_t=x)=\lim_{h\to0}\frac{1}{h}\frac{\Pr(Y_t=x|Y_{t+h}=y)\Pr(Y_{t+h}=y)}{\Pr(Y_t=x)}\\
        &=\lim_{h\to0}\frac{1}{h}\frac{q^\psi_{t+h}(y)}{q^\psi_t(x)}(\ul_{t+h}(x,y)h+o(h))=\frac{q^\psi_{t}(y)}{q^\psi_t(x)}\ul_t(x,y)=\frac{(q^\psi_t/\psi_t)(y)}{(q^\psi_t/\psi_t)(x)}r_t(y,x).
    \end{align*}
\end{remark}

\subsubsection{Proof of Part (1) of \cref{thm:cvg}}
\begin{proof}
Let $(Y_t)_{t\in[0,1]}\sim q^{\varphih^{(k-1)}_1}=:q$. By \cref{eq:kfe},
\begin{align*}
    \partial_tq_t(x)&=\sum_y q_t(y)\ur_t(x,y)=\sum_{y\ne x}(q_t(y)\ur_t(x,y)-q_t(x)\ur_t(y,x))\\
    &=\sum_{y\ne x}\Big(\frac{(q_t/\psi_t)(x)}{(q_t/\psi_t)(y)}r_t(x,y)q_t(y)-\frac{(q_t/\psi_t)(y)}{(q_t/\psi_t)(x)}r_t(y,x)q_t(x)\Big),\\
    \implies\partial_t\log q_t(x)&=\sum_{y\ne x}\Big(\frac{\psi_t(y)}{\psi_t(x)}r_t(x,y)-\frac{(q_t/\psi_t)(y)}{(q_t/\psi_t)(x)}r_t(y,x)\Big).
\end{align*}

On the other hand, using \cref{eq:kfe} again,
\begin{align*}
    \partial_t\psi_t(x)&=\sum_y\partial_tp^r_{t|0}(x|y)\psi_0(y)=\sum_y\sum_z r_t(x,z)p^r_{t|0}(z|y)\psi_0(y)=\sum_z r_t(x,z)\psi_t(z)\\
    &=\sum_{y\ne x}(r_t(x,y)\psi_t(y)-r_t(y,x)\psi_t(x)),\\
    \implies\partial_t\log \psi_t(x)&=\sum_{y\ne x}\Big(r_t(x,y)\frac{\psi_t(y)}{\psi_t(x)}-r_t(y,x)\Big).
\end{align*}

Now we compute $\KL(p^u\|q)$ where $p^u$ is the path measure of a CTMC $(X_t)_{t\in[0,1]}$ induced by transition rate $u_t$ and initial distribution $\mu$:
\begin{align*}
    \KL(p^u\|q)&=\KL(\mu\|q_0)+\E_{p^u(X)}\int_0^1\sum_{y\ne X_t}\Big(u_t\log\frac{u_t}{\ur_t}+ \ur_t - u_t\Big)(y,X_t)\d t\\
    &=\KL(\mu\|q_0)+\E_{p^u(X)}\int_0^1\sum_{y\ne X_t}\Big[
        \Big(u_t\log\frac{u_t}{r_t}+ r_t-u_t\Big)(y,X_t)\\
    &+\Big(-u_t(y,X_t)\log\frac{(q_t/\psi_t)(y)}{(q_t/\psi_t)(X_t)}+\frac{(q_t/\psi_t)(y)}{(q_t/\psi_t)(X_t)}r_t(y,X_t)-r_t(y,X_t)\Big)
        \Big]\d t.
\end{align*}
The second term is $\KL(p^u\|p^r)$. To deal with the third term, we leverage \cref{lem:ctmc_jump_expectation}:
\begin{align*}
    \E_{p^u(X)}\log\frac{(\psi_1/q_1)(X_1)}{(\psi_0/q_0)(X_0)}&=\E_{p^u(X)}\Big[\int_0^1(\partial_t\log\psi_t(X_t)-\partial_t\log q_t(X_t))\d t+\sum_{t:X_{t_-}\ne X_t}\log\frac{(\psi_t/q_t)(X_t)}{(\psi_t/q_t)(X_{t_-})}\Big]\\
    &=\E_{p^u(X)}\int_0^1\sum_{y\ne X_t}\Big(\frac{(q_t/\psi_t)(y)}{(q_t/\psi_t)(X_t)}r_t(y,X_t)-r_t(y,X_t)\Big)\d t\\
    &+\E_{p^u(X)}\int_0^1\sum_{y\ne X_t}u_t(y,X_t)\log\frac{(\psi_t/q_t)(y)}{(\psi_t/q_t)(X_t)}\d t.
\end{align*}
Therefore, we conclude that
\begin{align*}
    \KL(p^u\|q)&=\KL(p^u\|p^r)+\E_{p^u(X)}\log\frac{(\psi_1/q_1)(X_1)}{(\psi_0/q_0)(X_0)}+\const\\
    &=\KL(p^u\|p^r)+\E_{p^u(X)}\log\frac{\psi_1}{q_1}(X_1)+\const\\
    &=\KL(p^u\|p^r)+\E_{p^u(X)}\log\frac{\varphih_1^{(k-1)}}{\nu}(X_1)+\const,
\end{align*}
where $\const$ does not depend on $u$. Therefore, this is an SOC problem with terminal cost $g\gets\log\frac{\varphih_1^{(k-1)}}{\nu}$. Let the optimal path measure to this SOC problem be $p^{(k)}$ and, by relating this SOC problem with an equivalent SB problem through \cref{thm:soc_to_sb}, let the SB potentials to this problem be $(\varphi^{(k)}_t,\varphih^{(k)}_t)$. Then, $\varphi^{(k)}_1=\e^{-g}=\frac{\nu}{\varphih^{(k-1)}_1}$ by \cref{thm:soc_to_sb}.
From \cref{app:theory_sb} and similar argument as \cref{eq:varphi_ratio_tm}, we can leverage the additive noise and obtain
\begin{align*}
    \frac{\varphi^{(k)}_t(y)}{\varphi^{(k)}_t(x)}&=\E_{p^{(k)}_{1|t}(x_1|x)}\frac{\varphi^{(k)}_1(x_1+y-x)}{\varphi^{(k)}_1(x_1)}=\E_{p^{(k)}_{1|t}(x_1|x)}\frac{(\nu/\varphih^{(k-1)}_1)(x_1+y-x)}{(\nu/\varphih^{(k-1)}_1)(x_1)}.
\end{align*}

On the other hand, from the property of Bregman divergence, the unique fixed-point of \cref{eq:loss_varphi_tsm}, $\varPhi^{(k)}$, satisfies
\begin{align*}
    \varPhi^{(k)}_t(x)_{d,n}&=\E_{p^{r\varPhi^{(k)}}_{1|t}(x_1|x)}\frac{(\nu/\varphih^{(k-1)}_1)(x^{d\gets x_1^d+n-x^d})}{(\nu/\varphih^{(k-1)}_1)(x)}.
\end{align*}

Thus, we conclude that $\varPhi^{(k)}_t(x)_{d,n}=\frac{\varphi^{(k)}_t(x^{d\gets n})}{\varphi^{(k)}_t(x)}$.
\end{proof}

\subsubsection{Proof of Part (2) of \cref{thm:cvg}}
\begin{proof}
By similar arguments using Bayes formula, we can write $p^{r\varPhi^{(k)}}=:p$ as a backward CTMC initialized at $p_1$ with backward transition rate 
$$u^{\gets(k)}_t(y,x)=\frac{(p^{(k)}_t/\varphi^{(k)}_t)(y)}{(p^{(k)}_t/\varphi^{(k)}_t)(x)}r_t(x,y)=\frac{\varphih^{(k)}_t(y)}{\varphih^{(k)}_t(x)}r_t(x,y),$$
where $\varphi^{(k)}_t$ and $\varphih^{(k)}_t$ are the SB potentials to the SOC problem in the proof of the first part. Then, for any backward CTMC $(Y_t)_{t\in[0,1]}\sim q$ initialized at $Y_1\sim\nu$ with backward transition rate $\ul_t$, we have
\begin{align*}
    \KL(p\|q)&=\KL(p_1\|\nu)+\E_{X\sim p}\int_0^1\sum_{y\ne X_t}\Big(u^{\gets(k)}_t\log\frac{u^{\gets(k)}_t}{\ul_t}+\ul_t-u^{\gets(k)}_t\Big)(y,X_t)\d t.
\end{align*}
Therefore, it is obvious that the optimal $\ul_t(y,x)$ is equal to $u^{\gets(k)}_t(y,x)=\frac{\varphih^{(k)}_t(y)}{\varphih^{(k)}_t(x)}r_t(x,y)$, i.e., the optimal $q$ is $q^{\varphih^{(k)}_1}$ (\cref{def:cvg_q_psi}).
By similar arguments as in the proof of \cref{eq:hatvarphi_kfe}, we have $\varphih_t(x)=\sum_yp^r_{t|0}(x|y)\varphih_0(y)$, and hence by definition the optimal $q$ to the backward half bridge problem is $q^{\varphih^{(k)}_1}$. On the other hand, by similar arguments as \cref{eq:hatvarphi_ratio_dm,eq:hatvarphi_ratio_tm},
\begin{align*}
    \frac{\varphih^{(k)}_1(z)}{\varphih^{(k)}_1(y)}=\E_{p^{(k)}_{t|1}(x|y)}\frac{p^r_{1|t}(z|x)}{p^r_{1|t}(y|x)}=\E_{p^{(k)}_{t|1}(x|y)}\frac{\varphih^{(k)}_0(x-y+z)}{\varphih^{(k)}_0(x)}=\frac{(\mu/\varphi^{(k)}_0)(x-y+z)}{(\mu/\varphi^{(k)}_0)(x)}.
\end{align*}
On the other hand, using the property of Bregman divergence, the unique fixed-point of \cref{eq:loss_hatvarphi_dsm} satisfies
\begin{align*}
    \frac{\varphih^{(k)}_1(x^{d\gets n})}{\varphih^{(k)}_1(x)}=\E_{p^{(k)}_{t|1}(x_0|x)}\frac{p^r_{1|t}(x_1^{d\gets n}|x)}{p^r_{1|t}(x_1|x)},
\end{align*}
while the unique fixed-point of \cref{eq:loss_hatvarphi_tsm} satisfies
\begin{align*}
    \frac{\varphih^{(k)}_1(x^{d\gets n})}{\varphih^{(k)}_1(x)}=\E_{p^{(k)}_{t|1}(x_0|x)}\frac{(\mu/\varphi^{(k)}_0)(x_0^{d\gets x_0^d+n-x_1^d})}{(\mu/\varphi^{(k)}_0)(x_0)}.
\end{align*}
By comparing the three equations above, the proof is complete.
\end{proof}

\subsection{Other Omitted Results and Proofs}
\label{app:theory_misc}

\begin{proposition}
    Assume $\mu=\punif$. Using the reference transition rate \cref{eq:ref_transition_rate}, initializing the controller to be one (i.e., let $\varPhi^{(0)}_t(x)_{d,n}=1$ for all $n\ne x^d$) and alternatively training the corrector and the controller is equivalent to initializing the corrector to be one (i.e., let $\varPhih^{(0)}(x)_{d,n}=1$ for all $n\ne x^d$) and alternatively training the controller and the corrector. 
    \label{prop:networks_init}
\end{proposition}

\begin{proof}
Suppose we let $\varPhi^{(0)}_t(x)_{d,n}=1$ for all $n\ne x^d$. Then the optimal corrector can be computed as follows:
\begin{align*}
    \varPhih^{(1)}(x_1)_{d,n}&=\E_{p^r_{t|1}(x|x_1)}\frac{p^r_{1|t}(x_1^{d\gets n}|x)}{p^r_{1|t}(x_1|x)}=\sum_x\frac{p^r_{t|1}(x|x_1)}{p^r_{1|t}(x_1|x)}p^r_{1|t}(x_1^{d\gets n}|x)\\
    &=\sum_x \frac{p^r_t(x)}{p^r_1(x_1)}p^r_{1|t}(x_1^{d\gets n}|x)=\frac{1}{p^r_1(x_1)}\sum_x p^r_t(x)p^r_{1|t}(x_1^{d\gets n}|x)=\frac{p^r_1(x_1^{d\gets n})}{p^r_1(x_1)}.
\end{align*}
Thus, under \cref{eq:ref_transition_rate}, when $p^r_0=\mu=\punif$, one has $p^r_1=\punif$, so $\varPhih^{(1)}(x_1)_{d,n}=1$ for all $n\ne x_1^d$.
\end{proof}

\begin{remark}
    We remark that this choice coincides the optimal corrector when assuming $p^r$ is memoryless, since under memoryless condition $\varphih_1\propto p^r_1$.
\end{remark}

\begin{lemma}
    For a CTMC $(X_t)_{t\in[0,1]}\sim p^u$ with transition rate $u_t$ and any function $g:[0,1]\times\cX\times\cX\to\R$, we have
    \begin{align*}
        \E_{p^u(X)}\sum_{t:X_{t_-}\ne X_t}g(t,X_{t_-},X_t)&=\E_{p^u(X)}\int_0^1\sum_{y\ne X_t}g(t,X_t,y)u_t(y,X_t)\d t.
    \end{align*}
    \label{lem:ctmc_jump_expectation}    
\end{lemma}
\begin{proof}
Consider the time-discretization:
$\Delta t=\frac{1}{N}$ and $t_n=n\Delta t$. Then,
\begin{align*}
    \E_{p^u(X)}\sum_{t:X_{t_-}\ne X_t}g(t,X_{t_-},X_t)&=\E_{p^u(X)}\sum_{n=0}^{N-1}1_{X_{t_n} \ne X_{t_{n+1}}} g(t_n, X_{t_n}, X_{t_{n+1}})+O(\Delta t)\\
    &=\sum_{n=0}^{N-1}\sum_{x,y}p^u_{t_n}(x)p^u_{t_{n+1}|t_n}(y|x)1_{x\ne y}g(t_n,x,y)+O(\Delta t)\\
    &=\sum_{n=0}^{N-1}\sum_{x}p^u_{t_n}(x)\sum_{y\ne x}u_{t_n}(y,x)g(t_n,x,y)\Delta t+O(\Delta t)\\
    &=\E_{p^u(X)}\int_0^1\sum_{y\ne X_t}g(t,X_t,y)u_t(y,X_t)\d t.
\end{align*}
\end{proof}

\section{Experimental Details and Additional Results}
\label{app:exp}

\subsection{Target Distributions}
\label{app:exp_target_dist}
We consider Ising and Potts models on a square lattice $\Lambda=[L]^2$ with $L$ sites per dimension. We write $i\sim j$ if $i,j\in\Lambda$ are adjacent on the lattice. For simplicity, we impose periodic boundary conditions in both the horizontal and vertical directions. Both target distributions can be written in the form of 
$$\nu(x)=\frac1Z\e^{-\beta E(x)},$$
where $E$ is the energy function (Hamiltonian), $\beta>0$ is the inverse temperature, and $Z=\sum_{x\in\cX}\e^{-\beta E(x)}$ is the partition function.

For Ising model with interaction parameter $J\in\R$ and external magnetic field $h\in\R$, the energy is
\begin{align*}
    E_\mathrm{Ising}(x)=-J\sum_{i\sim j}x^ix^j-h\sum_ix^i,~x\in\{\pm1\}^\Lambda.
\end{align*}
We keep $J=1$ and $h=0$ through out the experiments. For Potts model with $N$ states and interaction parameter $J\in\R$, the energy is
\begin{align*}
    H_\mathrm{Potts}(x)=-J\sum_{i\sim j}1_{x^i=x^j},~x\in[N]^\Lambda.
\end{align*}
We keep $J=1$ through out the experiments.
Finally, we can compute the discrete score of these two distributions as follows:
\begin{align*}
    \frac{\nu(x^{i\gets n})}{\nu(x)}=
    \begin{cases}
        \exp\ro{\beta(n-x^i)\ro{J\sum_{j:~j\sim i}x^j+h}},~\forall n\in\{\pm1\},&\text{Ising};\\
        \exp\ro{\beta J\ro{\sum_{j:~j\sim i}\ro{1_{x^j=n}-1_{x^j=x^i}}}},~\forall n\in[N],&\text{Potts}.
    \end{cases}
\end{align*}

\subsection{Implementation Details}
\label{app:exp_details}

\paragraph{Model Backbone}
We follow the implementation of MDNS \citep{zhu2025mdns} to use vision transformers (ViT, \citet{dosovitskiy2021an}) to serve as the backbone for the discrete diffusion model. In particular, we use the DeiT (Data-efficient image Transformers) framework \citep{touvron2021training} with 2-dimensional rotary position embedding \citep{heo2025rotary}, which better captures the 2-dimensional spatial structure of the Ising and Potts models. While MDNS's model only requires the position input $x\in[N]^D$, in our learning objectives, the controller also receives time input $t\in[0,1]$. Hence, we adopt the adaptive layer normalization (adaLN) mechanism in the DiT \citep{peebles2023scalable} and SiT \citep{ma2024sit} to deal with the time conditioning.
For learning $24\times24$ Ising model and $16\times16$ Potts model with $4$ states, we use a model with $6$ blocks, hidden dimensions $32$ and $4$ heads. The total number of parameters in the model is around $144\mathrm{k}$ (controller, with time conditioning) or $95\mathrm{k}$ (corrector, without time conditioning).

\paragraph{Training}
Among all the training tasks, we use the AdamW optimizer \citep{loshchilov2018decoupled} with a constant learning rate of $1\e-3$ ($\betah$) and $5\e-4$ ($\betac,\betal$). We always use exponential moving average (EMA) to stabilize the training, with a decay rate of $0.9999$. 
All experiments are trained on an NVIDIA RTX A6000 GPU.
For all distributions, we always use the generalized KL divergence\footref{fn:gen_kl} as the Bregman divergence, always use uniform time weight $w_t\equiv1$, and adpot the modified log-linear noise schedule with $\gamma=1$ and $\alpha=0.5$. We train for $5$ stages with $500$ steps for controller and $250$ steps for corrector. We keep a running buffer of size $512$ ($\betah$) or $4096$ ($\betac,\betal$), and resampling a batch of $128$ pairs of $(x_0,x_1)$ using the current controller every $20$ ($\betah$) or $10$ ($\betac,\betal$) gradient updates.
Finally, for $\betah$, we use $\mu=\punif$ as an initialization and adopt AM loss \cref{eq:loss_hatvarphi_tsm} for the corrector; for $\betac,\betal$, we observe that uniform initialization does not lead to good performance, and thus initialize $\mu$ as the zero-temperature distribution (i.e., $\unif\{\pm\one\}$ for Ising and $\unif\{\one,2\one,...,N\one\}$ for Potts, where $\one$ is the all-one vector). We use DM loss \cref{eq:loss_hatvarphi_dsm} for training the corrector as $\mu$ is not positive everywhere.

\paragraph{Generating Baseline and Ground Truth Samples}
We follow the implementation in existing literature \citep{guo2026proximal} for baselines.
For the learning-based baseline, we train LEAPS and MDNS on $24\times24$ Ising model and $16\times16$ Potts model for up to $150\mathrm{k}$ steps for each temperature. For MDNS, we apply the warm-up strategy \citep{zhu2025mdns} to initialize the training under $\betac$ from the pretrained checkpoint for $\betah$, and initialize the training under $\betal$ from the pretrained checkpoint for $\betac$.
The Metropolis-Hastings (MH) and Swendsen-Wang (SW) sampling exactly follows the implementation in \citet{guo2026proximal} that ensures sufficient mixing.

\paragraph{Evaluation}
We follow the procedure detailed in \citet[App. D.3, E.2]{zhu2025mdns} for the computation of \textbf{magnetization} and \textbf{2-point correlation error}.
To compute the empirical \textbf{energy Wasserstein-2 distance} for two set of samples $\{x_i\}$ and $\{y_j\}$, we obtain the energies of the datasets $\cE_1=\{E(x_i)\}$ and $\cE_2=\{E(y_j)\}$, and use the function $\mathtt{np.sqrt(ot.wasserstein\_1d(\cE_1,~\cE_2,~p=2))}$ from the POT package \citep{pot}.

\newcommand{\qh}{\widehat{q}}
\paragraph{Effective Sample Size (ESS)}
We follow the practice in existing literature on neural samples \citep{zhang2022path,holderrieth2025leaps,zhu2025mdns}.
For a batch of i.i.d. samples $\{x_i\}_{i\in[B]}$ from $p$, suppose we associate each sample $x_i$ with a weight $w_i=\frac{\qh(x_i)}{p(x_i)}$, where $\qh$ is the unnormalized probability density / mass function of a probability distribution $q$, then the (normalized) \textbf{effective sample size (ESS)} of the samples with respect to $q$ is defined as $\frac{(\frac1B\sum_iw_i)^2}{\frac1B\sum_iw_i^2}\in\sq{\frac1B,1}$.

\paragraph{Ablation Studies on TM v.s. DM (\cref{fig:abl_tm_dm})}
We use a model with $4$ blocks, hidden dimension $32$ and $4$ heads. We use a batch size of $256$ and in \cref{eq:loss_varphi_dsm,eq:loss_varphi_tsm}, we sample $32$ $t$'s following $\unif[0,1]$ for each pair of $(x_0,x_1)$. All three cases are trained under the same random seed for $5000$ steps, with a buffer size of $1024$ and resampling frequency $20$ with generation NFE $100$.

\paragraph{Ablation Studies on Noise Schedules (\cref{fig:abl_loglin,fig:abl_const})}
We use a model with $4$ blocks, hidden dimension $32$ and $4$ heads. We use a batch size of $64$ and in \cref{eq:loss_varphi_tsm}, we sample $8$ $t$'s following $\unif[0,1]$ for each pair of $(x_0,x_1)$. All runs are trained under the same random seed for $5$ stages with $200$ controller update steps and $200$ corrector update steps. The buffer size is $256$ and resampling frequency is $20$.

\subsection{Further Experimental Results}
\label{app:exp_further_res}

\paragraph{Ablation Study: Noise Schedule}
We provide further ablation study for the constant noise schedule $\gamma_t\equiv\gamma$ in \cref{fig:abl_const}. A similar trend happens as the modified log-linear schedule: the generated samples reach the best quality at around $\gamma\in[0.5,1]$.

\begin{figure}[ht]
    \centering
    \includegraphics[width=0.75\linewidth]{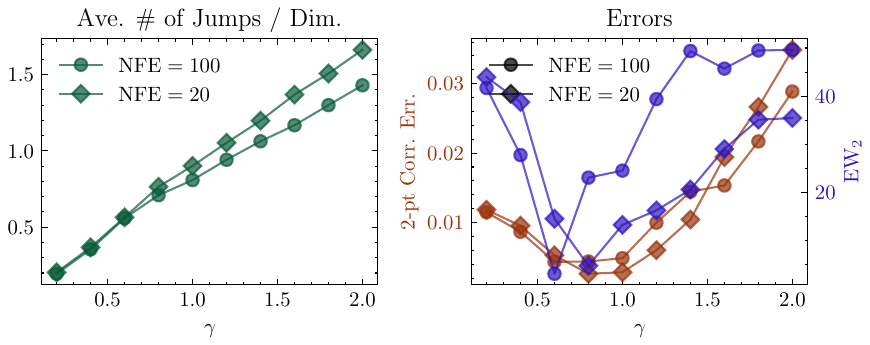}
    \caption{
    Ablation study of the hyperparameter $\gamma$ for the constant noise schedule $\gamma_t\equiv\gamma$ on Ising model with $L=24$ and $\betah=0.28$.
    NFE is the number of function evaluations during generation for both training and inference.
    \textit{Left}: average number of jumps for each dimension during generation.
    \textit{Right}: 2-point correlation error and energy Wasserstein-2 distance to ground-truth samples drawn from SW algorithm.
    }
    \label{fig:abl_const}
\end{figure}

\end{document}